%% file: sample_paper.tex
\documentclass[twoside]{article}

\usepackage{multicol}
\usepackage{microtype}
\usepackage{graphicx}
\usepackage{subfigure}
\usepackage{natbib}
\usepackage{epsfig,epsf,fancybox}
\usepackage{amsmath}
\usepackage{mathrsfs}
\usepackage{amssymb}
\usepackage{color}
\usepackage{xcolor}
\usepackage{multirow}
\usepackage{paralist}
\usepackage{verbatim}
\usepackage{algorithm}
\usepackage{algorithmic}
\usepackage{galois}
\usepackage{boxedminipage}
\usepackage{accents}
\usepackage{stmaryrd}
\usepackage[bottom]{footmisc}
\definecolor{light-gray}{gray}{0.85}
\usepackage{colortbl}
\usepackage{makecell}
\usepackage{hhline}
\usepackage{bbm}
\usepackage{mathtools}
\usepackage{xcolor}         
\usepackage{bbold}
\usepackage[utf8]{inputenc} 
\usepackage[T1]{fontenc}    
\usepackage{hyperref}       
\usepackage{url}            
\usepackage{booktabs}       
\usepackage{amsfonts}       
\usepackage{nicefrac}       
\usepackage{microtype}      
\usepackage{MnSymbol, graphicx}
\usepackage{tikz}
\usepackage{pifont}
\newcommand{\xmark}{\ding{55}}%

\usepackage{xr}

\usepackage[accepted]{aistats2022}

\newtheorem{theorem}{Theorem}[section]

\newtheorem{lemma}[theorem]{Lemma}
\newtheorem{proposition}[theorem]{Proposition}

\newtheorem{definition}{Definition}[section]

\newtheorem{remark}[theorem]{Remark}
\newtheorem{assumption}[theorem]{Assumption}

\newcommand{\norm}[1]{\left\lVert#1\right\rVert}
\newcommand\inner[2]{\langle #1, #2 \rangle}

\newcommand{\EE}{\mathbb{E}}

\newcommand{\ba}{\begin{array}}
\newcommand{\ea}{\end{array}}

\def\checkmark{\tikz\fill[scale=0.4](0,.35) -- (.25,0) -- (1,.7) -- (.25,.15) -- cycle;} 

\begin{document}

%
\runningtitle{On the Global Convergence of Momentum-based Policy Gradient}

%
\runningauthor{Yuhao Ding, Junzi Zhang, Javad Lavaei}

\twocolumn[

\aistatstitle{On the Global Optimum Convergence of \\
Momentum-based Policy Gradient}

\aistatsauthor{Yuhao Ding \And Junzi Zhang  \And Javad Lavaei }

\aistatsaddress{University of California, Berkeley \\ \texttt{yuhao\_ding@berkeley.edu}
\And  Amazon Advertising\\
\texttt{junziz@amazon.com}
\And University of California, Berkeley\\
  \texttt{lavaei@berkeley.edu}
  } ]
 \addtocounter{footnote}{0}

\begin{abstract}
Policy gradient (PG) methods are popular and efficient for large-scale reinforcement learning due to their relative stability and incremental nature. In recent years, the empirical success of PG methods has led to the development of a theoretical foundation for these methods. In this work, we generalize this line of research by establishing the first set of global convergence results of stochastic PG methods with momentum terms, which have been demonstrated to be efficient recipes for improving PG methods. We study both the soft-max and the Fisher-non-degenerate policy parametrizations, and show that adding a momentum term improves the global optimality sample complexities of vanilla PG methods by $\tilde{\mathcal{O}}(\epsilon^{-1.5})$ and $\tilde{\mathcal{O}}(\epsilon^{-1})$, respectively, where $\epsilon>0$ is the target tolerance. Our results for the generic Fisher-non-degenerate policy parametrizations also provide the first single-loop and finite-batch PG algorithm achieving an $\tilde{O}(\epsilon^{-3})$ global optimality sample complexity. Finally, as a by-product, our analyses provide general tools for deriving the global convergence rates of stochastic PG methods, which can be readily applied and extended to other PG estimators under the two parametrizations.
\end{abstract}

\input{sec/Intro}
\input{sec/prelim}

\input{sec/trajectory-based_PG}
\input{sec/Main_result}
\input{sec/conclu}

\section*{Acknowledgements}
This work was supported by grants from AFOSR, ARO, ONR, NSF and C3.ai Digital Transformation Institute.
Junzi Zhang did this work prior to joining or outside of Amazon.
We thank the anonymous reviewers for the helpful comments.

\bibliographystyle{apalike}
\bibliography{ref}
\onecolumn
\input{sec/related_work}
\input{sec/app_MBPG}
\input{sec/appex_softmax_relative_entropy}
\input{sec/app_compatible}

\end{document}

%% file: sec/intro.tex
\section{Introduction}\label{sec:Intro}



Policy gradient methods can be dated back to the pioneering work \cite{williams1992simple},  and have evolved into a rich family of reinforcement learning (RL) algorithms \citep{konda2000actor, kakade2001natural, silver2014deterministic, schulman2015trust, lillicrap2015continuous, schulman2017proximal}. In recent years, due to their amenability to function approximation and the development of deep neural networks, they have been successfully applied to a wide range of problems with significant empirical success, including robotic control, game playing, natural language processing, neural architecture search, and operations research \citep{zoph2016neural, silver2016mastering, yi2018neural, khan2020graph, wu2020cooperative}.

Momentum techniques have been demonstrated as a powerful and generic recipe for accelerating stochastic gradient methods, especially for nonconvex optimization and deep learning \citep{qian1999momentum, kingma2014adam, reddi2019convergence}. 
Recent works have also extended momentum techniques to improve policy gradient methods \citep{xiong2020non, yuan2020stochastic, pham2020hybrid, huang2020momentum}. As a state-of-the-art variance reduction technique, the momentum-based PG methods have been shown to outperform non-momentum methods such as SVRPG \citep{papini2018stochastic}, SRVR-PG  \citep{xu2019sample} and HAPG \citep{shen2019hessian} in practice.
In particular, \cite{xiong2020non} studies Adam-based policy gradient methods, but only achieves $O(\epsilon^{-4})$ sample complexities, which is the same as the one for the vanilla REINFORCE algorithm.
Inspired by the STORM algorithm for stochastic optimization in \cite{cutkosky2019momentum}, a new STORM-PG method is proposed in \cite{yuan2020stochastic}, which incorporates momentum in the updates and matches the sample complexity as the SRVR-PG method proposed in \cite{xu2019sample} (and also VRMPO) while requiring only single-loop updates and large initialization batches, whereas SRVR-PG and VRMPO require double-loop updates and large batch sizes throughout all iterations.
Concurrently, \cite{pham2020hybrid} proposes a hybrid estimator combining the momentum idea with SARAH and considers a more general setting with regularization, and achieves the same $O(\epsilon^{-3})$ sample complexity and again with single-loop updates and large initialization batches. 
Finally, independently inspired by the STORM algorithm in \cite{cutkosky2019momentum}, \cite{huang2020momentum} proposes a class of momentum-based policy gradient algorithms with  adaptive time-steps, single-loop updates and small batch sizes, which match the sample complexity as in \cite{xu2019sample}. 
However, all the above  sample complexity results for momentum-based policy gradient methods only apply to convergence to a first-order stationary point, which may have an arbitrarily poor performance. 
The global optimality sample complexity of momentum-based stochastic PG methods remains an open question. 

\begin{table*}[ht]
\centering
\label{table: comparison}
{\small
\renewcommand{\arraystretch}{1.5}
 \begin{tabular}{c| c|c |c |c} 
 \hline
 \textbf{Algorithm} & \textbf{Parametrization} & \textbf{Complexity} & \textbf{Single Loop} & \textbf{Finite Batch} \\\hline\hline
   PG \citep{wang2019neural} & neural & $\tilde{O}(\epsilon^{-4})$ & \checkmark & \xmark\\\hline
    PG \citep{liu2020improved} & Fisher-non-degenerate & $\tilde{O}(\epsilon^{-4})$ & \checkmark & \xmark \\\hline
 PG \citep{zhang2020sample}& soft-max & $\tilde{O}(\epsilon^{-6})$ &  \checkmark&\checkmark \\
 \hline\hline
   SRVR-PG \citep{liu2020improved} & Fisher-non-degenerate & $\tilde{O}(\epsilon^{-3})$ & \xmark & \xmark \\\hline
         Ours & Fisher-non-degenerate & {$\tilde{O}(\epsilon^{-3})$} & {\checkmark} & {\checkmark}\\\hline 
   Ours & soft-max & {$\tilde{O}(\epsilon^{-4.5})$} & {\checkmark} & {\checkmark}\\\hline 
 \end{tabular}
 \caption{We summarize comparisons with the existing global optimum convergence results for policy gradient methods. The (sample) complexity is defined as the number of trajectories needed to reach the global sub-optimality gap $\epsilon>0$ plus some inherent function approximation error (if any), and we ignore logarithmic terms.  Note that ``single loop'' only refers to the policy update step, and does not refer to the policy evaluation part (for actor-critic versions of PG).  
 }
 }
\end{table*}

Inspired by recent advances in the global optimum convergence theory of PG methods \citep{agarwal2020optimality, zhang2020sample, zhang2021convergence, liu2020improved}, we address the aforementioned problem in this paper.  We focus on the study of STORM-based PG method introduced in \cite{huang2020momentum} due to its sample efficiency and the simplicity of the algorithm, and provide the first global optimum convergence results for the momentum-based PG. We summarize our detailed contributions below.
\begin{itemize}
    \item For the soft-max policy parameterization, we show that adding momentum terms improves the existing global optimality sample complexity bounds of PG in \cite{zhang2020sample} by $\tilde{\mathcal{O}}(\epsilon^{-1.5})$.
    \item For the generic  Fisher-non-degenerate policy parametrization, we show that adding momentum terms improves the existing sample complexity bounds of PG in \cite{liu2020improved} by $\tilde{\mathcal{O}}(\epsilon^{-1})$  and matches the sample complexity bounds of SRVR-PG in \cite{liu2020improved}. Our result is the first single-loop and finite-batch policy gradient algorithm achieving $\tilde{O}(\epsilon^{-3})$ global optimality sample complexity.
    \item As a by-product, our analyses also summarize general tools (\textit{cf}. Lemmas \ref{lemma: relative softmax upper bound on I^+} and \ref{lemma: restricted J start -J t}) for deriving global optimality sample complexities of stochastic policy gradient methods with  soft-max and  Fisher-non-degenerate policies, which can be easily applied and extended to different policy gradient estimators under these parametrizations.
\end{itemize}
Comparisons with the existing global optimum convergence results of policy gradient methods can be found in Table \ref{table: comparison}. Due to space restrictions, we provide a  more detailed literature review and introduce the notations in Sections \ref{sec: related work} and \ref{sec: Notation} of the appendix. 

%% file: sec/prelim.tex
\section{Preliminaries}\label{sec:prelim}
\subsection{Reinforcement learning}
Reinforcement learning is generally modeled as a discounted Markov decision process (MDP) defined by a tuple $(\mathcal{S},\mathcal{A}, \mathbb{P},r,\gamma)$, where $\mathcal{S}$ and $\mathcal{A}$ denote the finite state and action spaces, $\mathbb{P}(s^\prime|s,a)$ is the probability that the agent transits from the state $s$ to the state $s^\prime$ under the action $a\in \mathcal{A}$. $r(s,a)$ is the reward function, i.e., the agent obtains the reward $r(s_h,a_h)$ after it takes the action $a_h$ at the state $s_h$ at time $h$. We also assume that the reward is bounded, i.e., $r(s,a):\mathcal{S}\times \mathcal{A} \rightarrow [0,1]$.
$\gamma\in (0,1)$ is the discount factor. The policy $\pi(a|s)$ at the state $s$ is usually represented by a conditional probability distribution $\pi_\theta(a|s)$ associated with the parameter $\theta\in \mathbb{R}^d$, where $d$ is the dimension of the parameter space. Let $\tau=\{s_0,a_0,s_1,a_1,\ldots\}$ denote the data of a sampled trajectory under policy $\pi_\theta$ with the probability distribution over trajectory as $$p(\tau|\theta,\rho)=\rho(s_0) \prod_{h=1}^{\infty}\mathbb{P}(s_{h+1}|s_h,a_h)\pi_\theta(a_h|s_h),$$
where $\rho \sim \Delta(\mathcal{S})$ is the probability distribution of the initial state $s_0$. Here,  $\Delta(\mathcal{X})$ denotes the probability simplex over a finite set $\mathcal{X}$. For every policy $\pi$, one can define the state-action value function $Q^\pi: \mathcal{S}\times \mathcal{A} \rightarrow \mathbb{R}$ as 
$$
Q^{\pi}(s,a)\coloneq\EE_{\hspace{-0.3cm}\substack{a_h\sim \pi(\cdot|s_h)\\ s_{h+1}\sim\mathbb{P}(\cdot|s_h,a_h)}} \left(\sum_{h=0}^\infty \gamma^h r(s_h,a_h) \bigg\rvert s_0=s,a_0=a \right).
$$
The state-value function $V^\pi: \mathcal{S}\rightarrow \mathbb{R}$ and the advantage function $A^\pi: \mathcal{S}\times \mathcal{A} \rightarrow \mathbb{R}$, under the policy $\pi$,  can be defined as
\begin{align*}
   & V^\pi(s)\coloneq \EE_{a\sim\pi(\cdot|s)}[Q^\pi(s,a)],\\
   & A^\pi(s,a)\coloneq Q^\pi(s,a)-V^\pi(s).
\end{align*} 
Then, the goal is to find an optimal policy in the policy class that maximizes the expected discounted return, namely,
\begin{align} \label{eq: original objective}
    \max_{\theta} \quad J_\rho(\pi_\theta)\coloneq \EE_{s_0 \sim \rho} [V^{\pi_\theta}(s_0)].
\end{align}
For notional convenience, we denote $J_\rho(\pi_\theta)$ by the shorthand notation $J_\rho(\theta)$ and also let $\theta^\ast $ denote a global maximum of $J_\rho(\theta)$. In practice, a truncated version of the value function is used to approximate the infinite sum of rewards in \eqref{eq: original objective}. Let
$$\tau_H=\{s_0,a_0,s_1,\ldots, s_{H-1}, a_{H-1}, s_H\}$$ 
denote the truncation of the full trajectory $\tau$ of length $H$. The truncated version of the value function is defined as $$J_\rho^H(\theta)\coloneq\EE_{\substack{s_0\sim\rho,  a_h\sim \pi_\theta(\cdot|s_h)\\ s_{h+1}\sim\mathbb{P}(\cdot|s_h,a_h)}} \left(\sum_{h=0}^{H-1} \gamma^h r(s_h,a_h) \bigg\rvert s_0\right).$$

\subsection{Discounted state visitation distributions}
The discounted state visitation distribution $d_{s_0}^\pi$ of a policy $\pi$ is defined as 
$$
d_{s_0}^\pi(s)\coloneq (1-\gamma) \sum_{h=0}^\infty \gamma^h \mathbb{P}(s_h=s|s_0,\pi),
$$
where $\mathbb{P}(s_h=s|s_0,\pi)$ is the state visitation probability that $s_h$ is equal to $s$ under the policy $\pi$ starting from the state $s_0$. Then, the discounted state visitation distribution under the initial distribution $\rho$ is defined as $d_\rho^\pi(s)\coloneq\EE_{s_0\sim\rho} [d_{s_0}^\pi(s)]$.
Furthermore, the state-action visitation distribution induced by $\pi$ and the initial state distribution $\rho$ is defined as $v_\rho^{\pi}(s,a)\coloneq d_{\rho}^\pi(s)\pi(a|s)$, which can also be written as 
$$ v_\rho^{\pi}(s,a) \coloneq  (1-\gamma) \EE_{s_0\sim\rho}\sum_{h=0}^\infty \gamma^h \mathbb{P}(s_h=s,a_h=a|s_0,\pi),$$
where $\mathbb{P}(s_h=s,a_h=a|s_0,\pi)$ is the state-action visitation probability that $s_h=s$ and $a_h=a$ under  $\pi$ starting from the state $s_0$.


\subsection{Policy parameterization}
In this work, we consider the following two different policy classes: 

 \textbf{Soft-max parameterization}. For an unconstrained parameter $\theta \in \mathbb{R}^{|\mathcal{S}||\mathcal{A}|}$, the policy $\pi_\theta(a|s)$ is chosen to be 
 $$
 \frac{\exp{(\theta_{s,a})}}{\sum_{a^\prime \in \mathcal{A}}\exp{(\theta_{s,a^\prime}})}.
 $$ 
 The soft-max parameterization is generally used for Markov Decision Processes (MDPs) with finite state and action spaces. It is complete in the sense that every stochastic policy can be represented by this class.
    
 \textbf{Fisher-non-degenerate parameterization}.
    We study the general policy class that satisfies Assumption \ref{ass: fisher info} given below: 
    \begin{assumption} \label{ass: fisher info}
    For all $\theta \in \mathbb{R}^d$, there exists some constant $\mu_F>0$ such that the Fisher information matrix $F_\rho(\theta)$ induced by the policy $\pi_\theta$ and the initial state distribution $\rho$ satisfies
    \begin{align*}
        F_\rho(\theta)=\EE_{(s,a)\sim v_\rho^{\pi_\theta}}[\nabla \log \pi_\theta(a|s) \nabla \log \pi_\theta(a|s)^\top]\succcurlyeq \mu_F\cdot I_d.
    \end{align*}
    \end{assumption}
 Assumption \ref{ass: fisher info}, which is also used in \cite{liu2020improved}, essentially states that $F_\rho(\theta)$ is well-behaved as a pre-conditioner in the natural PG update \citep{kakade2002approximately}.
It is shown in \cite{liu2020improved} that the positive definiteness of $F_\rho(\theta)$ in Assumption \ref{ass: fisher info} can be satisfied by certain Gaussian policies, where $\pi_\theta(\cdot|s)=\mathcal{N}(\mu_\theta(s),\Sigma)$ with the parametrized mean function $\mu_\theta(s)$ and the fixed covariance matrix $\Sigma \succ 0$, provided that the Jacobian of $\mu_\theta(s)$ is full-row rank for all $\theta \in \mathbb{R}^d$. In addition, Assumption \ref{ass: fisher info} holds more generally for every full-rank exponential family paramertrization with their mean parameterized by $\mu_\theta(s)$ if $\mu_\theta(s)$ is full-row rank for all $\theta \in \mathbb{R}^d$. We refer the reader to Section \ref{sec: app-FNP} of appendix for more discussions.

It is worth noting that Assumption \ref{ass: fisher info} is not satisfied by the soft-max parameterization when $\pi_\theta$ approaches a deterministic policy, which means that the two policy parameterizations to be studied here do not overlap.

%% file: sec/trajectory-based_PG.tex
\section{Trajectory-based policy gradient estimator}
The policy gradient method \citep{sutton2018reinforcement} is one of the standard ways to solve the optimization problem \eqref{eq: original objective}. Since the distribution $p(\tau|\theta)$ is unknown, $\nabla J_{\rho}(\theta)$ needs to be estimated from samples.  Then, a stochastic PG ascent update with the exploratory initial distribution $\rho$ at time step $t$ is given as 
\begin{align}\label{eq: PG update}
    \theta_{t+1}=\theta_t+ \frac{\eta_t}{B} \sum_{i=1}^{B} u_{t,i},
\end{align}
where $\eta_t>0$ is the learning rate, $B$ is the batch size of trajectories, and  $u_{t,i}$ can be any PG estimator of $\nabla J_{\rho}(\theta_t)$. 
If the  parameterized policy satisfies Assumption \ref{ass: 1} to be stated later and the reward function is not dependent on the parameter $\theta$, PG estimators can be obtained from a single sampled trajectory. These trajectory-based estimators
include REINFORCE \citep{williams1992simple}, PGT \citep{sutton1999policy}  and GPOMDP \citep{baxter2001infinite}. Compared with PG estimators based on the state-action visitation measure \citep{agarwal2020optimality}, the trajectory-based PG estimators are often used in practice due to their sample efficiency and amenability to using the importance sampling for variance reduction.  
In practice, the truncated versions of these trajectory-based PG estimators are used to approximate the infinite sum in the PG estimator. 
For example,
the commonly used truncated REINFORCE with a {constant baseline} $b$ is given by
\begin{align*}
    g(\tau_H^i|\theta,\rho)=\left(\sum_{h=0}^{H-1} \nabla \log \pi_\theta(a_h^i,s_h^i) \right)\left(\sum_{h=0}^{H-1} \gamma^h r_h(s^i_h,a^i_h)-b\right),
\end{align*}
where the superscript $i$ is the index of the trajectories.
The commonly used truncated  PGT is given by:
\begin{align*} 
 g(\tau_H^i|\theta,\rho)=\sum_{h=0}^{H-1}\sum_{j=h}^{H-1} \nabla \log \pi_\theta(a_h^i,s_h^i) \left( \gamma^j r_j(s^i_j,a^i_j)\right).
\end{align*}
The PGT estimator is also equivalent to the popular truncated GPOMDP estimator defined as follows
\begin{align}\label{eq: GPOMDP}
g(\tau_H^i|\theta,\rho)=\sum_{h=0}^{H-1}\sum_{j=0}^{h} \nabla \log \pi_\theta(a_j^i,s_j^i) \left( \gamma^h r_h(s^i_h,a^i_h)-b_h\right).
\end{align}
We first make the following essential assumption for PG estimators.
 \begin{assumption} \label{ass: 1}
The gradient and hessian of the function $\log \pi_\theta(a|s)$ are bounded, i.e., there exist constants $M_g, M_h>0$ such that
$\norm{\nabla \log\pi_\theta(a|s)}_2\leq M_g$ and $\norm{\nabla^2 \log\pi_\theta(a|s)}_2\leq M_h$ for all $\theta \in \mathbb{R}^d$.
\end{assumption}


For the soft-max parameterization, Assumption \ref{ass: 1} is satisfied with $M_g=2$ and $M_h=1$ (see Lemma \ref{lemma: mg and mh for softmax} in appendix).
Assumption \ref{ass: 1} has also been commonly used in the analysis of the policy gradient \citep{papini2018stochastic,xu2019sample, xu2020improved,shen2019hessian, liu2020improved, huang2020momentum} for the more general policy parameterization.  Although Assumption \ref{ass: 1} is satisfied for soft-max policies and more generally, log-linear policies with bounded feature vectors (\textit{cf}. Section 6.1.1 in \cite{agarwal2020optimality}), it fails for very common policies such as Gaussian policies. We believe this assumption  could be relaxed if the truncated policy gradient is used to guarantee the boundedness of the importance sampling weight \cite{zhang2020global} and we leave it as a future work.
Based on Assumption \ref{ass: 1}, we provide some useful properties of stochastic policy gradient and the value function.


\begin{proposition} \label{lemma:Properties of PGT estimator}
For the truncated GPOMDP policy gradient given in \eqref{eq: GPOMDP} satisfying Assumptions \ref{ass: 1},
the following properties hold for all initial distribution $\rho$ and for all $\theta \in \mathbb{R}^d$:
\begin{enumerate}
     \item { $g(\tau_H|\theta,\rho)$ is $L_g$-Lipschitz continuous, where $L_g\coloneq M_h/(1-\gamma)^2$.
 \item $\norm{g(\tau_H|\theta,\rho)}_2 \leq G$, where $G \coloneq M_g/(1-\gamma)^2$.
    \item $\text{Var}(g(\tau_H|\theta,\rho))\leq \sigma^2$, where $\sigma \coloneq G$.}
  \item $J_\rho(\theta,\rho )$ and $J^H_\rho(\theta )$ are $L$-smooth, namely, $\max\{ \norm{ \nabla^2 J_\rho(\theta )}_2, \norm{ \nabla^2 J^H_\rho(\theta )}_2\}\leq L$, where $L\coloneq\frac{2 M_g^{2}}{(1-\gamma)^{3}}+\frac{M_h }{(1-\gamma)^{2}}$.
    \item If the infinite-sum is well-defined, then 
    \begin{align*} 
    &g(\tau_H|\theta,\rho)=\\
    &\sum_{h=0}^{\infty}\sum_{j=h}^{\infty} \nabla \log \pi_\theta(a_h,s_h) \left( \gamma^j r_j(s_j,a_j)-b_j\right).
\end{align*}
 is an unbiased estimate of $\nabla J_\rho(\theta)$. Similarly, the truncated GPOMDP estimate $g(\tau_H|\theta,\rho)$ given by \eqref{eq: GPOMDP} is an unbiased estimate of $\nabla J_\rho^H(\theta)$.
     \item $\norm{\nabla J_\rho^H(\theta) -\nabla J_\rho(\theta)}_2 \leq M_g \left(\frac{H+1}{1-\gamma} +\frac{\gamma}{(1-\gamma)^2}\right) \gamma^H$.
     \item $\max\{ \norm{ \nabla J_\rho(\theta )}_2, \norm{ \nabla J^H_\rho(\theta )}_2\}\leq G$.
\end{enumerate}
\end{proposition}

The first two properties are shown in Proposition 4.2 in \cite{xu2019sample}. 
The third and forth properties follow from Lemma 4.2 and Lemma 4.3 in \cite{yuan2021general}, respectively.
The last three properties follow directly from Lemma B.1 in \cite{liu2020improved}. 

\subsection{Momentum-based policy gradient}
Due to the high sample complexity of the vanilla PG, many recent works have turned onto variance reduction methods for PG, including the momentum-based policy gradient \citep{huang2020momentum,yuan2020stochastic}. 
The momentum-based policy gradient  with the batch size of $B$ and the sampled trajectory of length $H$ is defined as 
\begin{align} \label{eq: momentum-based PG}
  & u_t^H= \frac{\beta_t}{B} \sum_{i=1}^B g(\tau_H^i|\theta_t,\rho) +(1-\beta_t)\left[u_{t-1}^H+\right.\\
 \nonumber  &\left.\frac{1}{B} \sum_{i=1}^B\left(g(\tau_H^i|\theta_t,\rho) - w(\tau_H^i|\theta_{t-1},\theta_t)g(\tau_H^i| \theta_{t-1},\rho)\right)\right]
\end{align}
for all $t\in\{2,\ldots,T\}$, 
where $g(\tau_H^i|\theta_t,\rho)$ is the vanilla PG estimator such as \eqref{eq: GPOMDP},
$\beta_t\in[0,1]$, and the importance sampling weight is defined as
\begin{align}\label{eq: importance sample weight}
    w(\tau_H|\theta^\prime,\theta) =\frac{p(\tau_H|\theta^\prime,\rho)}{p(\tau_H|\theta,\rho)}=\prod_{h=0}^{H-1}\frac{\pi_{\theta^\prime}(a_h|s_t)}{\pi_{\theta}(a_h|s_t)}.
\end{align}
This importance sampling weight guarantees that 
\begin{align*}
&\EE_{\tau_H\sim p(\cdot|\theta,\rho)}[g(\tau_H|\theta,\rho)  - w(\tau_H|\theta^\prime, \theta) g(\tau_H|\theta^\prime,\rho)]\\
&=\nabla J_\rho^H(\theta)-\nabla J_\rho^H(\theta^\prime).
\end{align*}
Then, by carefully choosing $\eta_t$ and $\beta_t$, the accumulated policy gradient estimation error $u_t^H -\nabla J_\rho^H(\theta_t)$ can be well controlled.  To guarantee the convergence of the momentum-based policy gradient, we require the following assumption:

\begin{assumption} \label{ass: 3}
{
For every $\theta_t$ and $\theta_{t+1}$ satisfying \eqref{eq: PG update}, the variance of 
$w(\tau_H|\theta_t,\theta_{t+1})$ is bounded, i.e., there exists a constant $W>0$ such that $\text{Var}(w(\tau_H|\theta_t,\theta_{t+1})) \leq W$ for all $\tau_H\sim p(\cdot|\theta_{t+1},\rho)$.}
\end{assumption}
Assumption \ref{ass: 3} has been commonly used in the analysis of some variance reduced variants of PG \citep{papini2018stochastic,xu2019sample, xu2020improved,shen2019hessian, liu2020improved, huang2020momentum}. 
It is worthwhile to note that the bounded importance sampling weight in Assumption \ref{ass: 3} may be violated in practice. 
A commonly used remedy to make the algorithm more effective is to clip the importance sampling weights \citep{huang2020momentum}. {In addition, when $\pi_\theta$ is the soft-max parameterization, the importance sampling weight $w(\tau_H|\theta_t,\theta_{t+1})$ in 
Assumption \ref{ass: 3} has a bounded variance by using the truncated policy gradient (see Lemma 5.6 in \cite{zhang2021convergence}).}

The key reason that the momentum can improve the convergence rates is that the momentum can help reduce the variance of the estimated stochastic gradient. To build some intuition for this, let $e_t=u_t^H-\nabla J^H_\rho(\theta_t)$. It can be verified that
$\mathbb{E}[e_t]=(1-\beta_t)\mathbb{E} [e_{t-1}]$ and
\begin{align*}
\mathbb{E} \left[||e_t||^2 \right]\leq& (1-\beta_t)^2\mathbb{E}\left[||e_{t-1}||^2\right]+2\beta_t^2\mathbb{E}\left[||T_1||^2\right]\\
&+2(1-\beta_t)^2\mathbb{E}\left[||T_2||^2\right],
\end{align*}

where 
\begin{align*}
\mathbb{E}\left[||T_1||^2 \right]&=\text{Var}(g(\tau_H|\theta_t))\leq \sigma^2,\\
\mathcal{O}\left(||T_2||^2\right)&=\mathcal{O}(||\theta_t-\theta_{t-1}||^2)=\mathcal{O}(\eta_t^2||u_t||^2).
\end{align*}
 Then, the variance of the stochastic gradient $u_t$ can be reduced with the appropriate choices of $\eta_t$ and $\beta_t$.

%% file: sec/Main_result.tex
\section{Global convergence of momentum-based policy gradient}\label{sec:Global convergence analysis}
As mentioned in the previous sections, the global convergence of policy gradient depends on the parameterization of the policy. In this section, we will study the global convergence and the sample complexity of the momentum-based policy gradient for both soft-max parameterization with a log barrier penalty and the more general parameterization satisfying the fisher-non-degenerate assumption. 

\subsection{Soft-max parameterization with log barrier penalty}\label{sec: Soft-max parameterization with log barrier penalty}
 \subsubsection{Preliminary tools}
 
We first study the global convergence of momentum-based policy gradient for the soft-max policy parameterization (Algorithm \ref{alg:IS-MBPG-S}), where $\pi_\theta(a|s)=\frac{\exp{(\theta_{s,a})}}{\sum_{a^\prime \in \mathcal{A}}\exp{(\theta_{s,a^\prime}})}$ for all $\theta \in \mathbb{R}^{|\mathcal{S}||\mathcal{A}|}$. 
Optimization over the soft-max parameterization is problematic since the optimal policy---that is usually deterministic---is obtained by letting some parameters grow towards infinity. To prevent the parameters from becoming too large and to ensure adequate exploration, a log-barrier regularization term that penalizes the policy for becoming deterministic is commonly used. The regularized objective is defined as
\begin{align*}
    L_{\lambda,\rho}(\theta)=&J_\rho(\theta)-\lambda \EE_{s\sim\text{Unif}_\mathcal{S}} \left[\text{KL}(\text{Unif}_\mathcal{A}, \pi_\theta(\cdot|s)) \right]\\
    =&J_\rho(\theta)+ \frac{\lambda}{|\mathcal{A}|| \mathcal{S}|}\sum_{s,a} \log \pi_\theta(a|s) +\lambda \log|\mathcal{A}|,
\end{align*}
where $\text{KL}(p,q)\coloneq \EE_{x\sim p}[-\log q(x)/p(x)]$ and $\text{Unif}_{\mathcal{X}}$ denotes the uniform distribution over a set $\mathcal{X}$.

\begin{algorithm}[ht] 
\caption{Momentum-based PG with soft-max parameterization (STORM-PG-S)}
\label{alg:IS-MBPG-S}
\begin{algorithmic}[1]
\STATE \textbf{Inputs}: Iteration $T$, horizon $H$, batch size $B$, initial input $\theta_1$, parameters $\{k,m,c\}$, initial distribution $\mu$;
\STATE \textbf{Outputs}: $\theta_\xi$ chosen uniformly random from $\{\theta_t\}_{t=1}^T$;
\FOR{$t = 1, 2, \dots,T-1$}
\STATE Sample $B$ trajectories $\{\tau^H_i\}_{i=1}^B$ from $p(\cdot|\theta_t,\mu)$;
\IF{$t=1$}
\STATE Compute $u_1^H=\frac{1}{B} \sum_{i=1}^B g(\tau^i|\theta_1,\mu)$;
\ELSE{}
\STATE Compute $u_t^H$ based on \eqref{eq: momentum-based PG};
\ENDIF{}
\STATE Compute $\eta_t =\frac{k}{(m+t)^{1/3}}$;
\STATE Update \\
$\theta_{t+1}=\theta_{t}+\eta_t (u_t^H+ \frac{\lambda}{|\mathcal{A}|| \mathcal{S}|}\sum_{s,a} \nabla \log \pi_\theta(a|s))$;
\STATE Update $\beta_{t+1}=c\eta_t^2$;
\ENDFOR{}
\end{algorithmic}
\end{algorithm}

In addition, while we are interested in the value $J_{\rho}(\theta)$ under the performance measure $\rho$, it may be helpful to optimize under the initial distribution $\mu$, i.e., the policy gradient is taken with respect to the optimization measure $\mu$, where $\mu$ is usually chosen as an exploratory initial distribution that adequately covers the state distribution of some optimal policy. 
Then, the notion of the distribution mismatch coefficient is defined below:
\begin{definition}
Given a policy $\pi$ and measures $\rho, \mu \in \Delta(\mathcal{S})$, the distribution mismatch coefficient of $\pi$ under $\rho$ relative to $\mu$ is defined as $\norm{\frac{d_\rho^\pi}{\mu}}_\infty$, where $\frac{d_\rho^\pi}{\mu}$ denotes componentwise division.
\end{definition}
It is shown in \cite{agarwal2020optimality} that the difficulty of the exploration problem faced by policy gradient algorithms can be captured through this distribution mismatch coefficient.

Although the optimization problem defined above is non-convex in general, Theorem 5.3 in \cite{agarwal2020optimality} has shown that the first-order stationary points of the regularized objective are approximately globally optimal solutions of $J_\rho(\theta)$ when the regularization parameter $\lambda$ is sufficiently small and the exact PG is available.

\begin{lemma}[\cite{agarwal2020optimality}]\label{prop: softmax, relative entropy, gradient dom}
Suppose that $\theta$ satisfies the inequality $\norm{\nabla L_{\lambda,\mu}(\theta)} \leq \epsilon_{opt}$ with $\epsilon_{opt} \leq \frac{\lambda}{2|\mathcal{S}| |\mathcal{A}|}$. Then, for every initial distribution $\rho$, we have: 
\begin{align*} 
  J_\rho(\theta^\ast)- J_\rho(\theta)\leq \frac{2\lambda}{1-\gamma} \norm{\frac{d_\rho^{\pi_{\theta^\ast}}}{\mu}}_\infty.
\end{align*}
\end{lemma}
\subsubsection{Theoretical results}
Motivated by the above result and the proof idea in \cite{zhang2020sample}, we can relate the global convergence to the convergence of the first-order stationary points of the regularized objective.

\begin{lemma} \label{lemma: relative softmax upper bound on I^+}
{Consider a soft-max parameterization  $\pi_\theta$. Given a fixed constant $\epsilon>0$, let $\lambda=\frac{\epsilon(1-\gamma)}{4\norm{\frac{d_\rho^{\pi_{\theta^\ast}}}{\mu}}_\infty}$. For every initial distribution $\rho$, every step-size sequence $\{\eta_t\}_{t=1}^T$, and every sequence $\{\theta_t\}_{t=1}^T$, we have: }
\begin{align*}
&J_\rho(\theta^\ast)-\frac{1}{T}\sum_{t=1}^T \EE \left[J_\rho(\theta_t)\right]\\
&\leq \norm{\frac{d_\rho^{\pi_{\theta^\ast}}}{\mu}}^2_\infty\frac{64|\mathcal{S}|^2 |\mathcal{A}|^2\sum_{t=1}^T\eta_t  \EE[ \norm{\nabla L_{\lambda,\mu}(\theta_t)}_2^2]}{\epsilon^2T\eta_T(1-\gamma)^3} +\frac{\epsilon}{2}.
\end{align*}
\end{lemma}

It is worth noting that the bound in Lemma \ref{lemma: relative softmax upper bound on I^+} is agnostic to the algorithms.
To prove Lemma \ref{lemma: relative softmax upper bound on I^+}, we first define the following set of ``bad'' iterates:
\begin{align*}
    I^+=\left\{t\in\{1,\ldots,T\} \Big|  \norm{\nabla_\theta L_{\lambda,\rho} (\theta_t)} \geq \frac{\lambda}{2|\mathcal{S}| |\mathcal{A}|}\right\}.
\end{align*}
which counts the number of iterates such that the gradient norms of the KL-regularized objective are large. We then carefully upper-bound the number of iterates in the set $I^+$ using the accumulated gradient norm. One can show that for every $\epsilon>0$ and $\lambda=\epsilon(1-\gamma)/(4\norm{\frac{d_\rho^{\pi_{\theta^\ast}}}{\mu}}_\infty)$, 
we have that
\begin{align*}
    &J_\rho(\theta^\ast)-J_\rho(\theta)\leq \frac{\epsilon}{2}, \quad \forall k\in \{0,\ldots,K\}/ I^+,\\
    & J_\rho(\theta^\ast)-J_\rho(\theta)\leq 1/(1-\gamma), \quad \forall k\in I^+,
\end{align*} 
where the second inequality is due to the assumption that the rewards are between 0 and 1. Finally, by combining with the first result, we obtain the desired bound. For the details of the proof, we refer the reader to the appendix in Section \ref{sec: Proof of Lemma relative softmax upper bound on I^+ }.

With Lemma \ref{lemma: relative softmax upper bound on I^+}, it remains to bound the accumulated stationary convergence of the stochastic policy gradient. Let $$L_{\lambda,\mu}^H(\theta_t)\coloneq J^H_\mu(\theta)+ \frac{\lambda}{|\mathcal{A}|| \mathcal{S}|}\sum_{s,a} \log \pi_\theta(a|s) +\lambda \log|\mathcal{A}|$$
be the regularized value function with the truncated horizon $H$. By applying the momentum-based PG, we arrive at the following result:

\begin{lemma} \label{lemma: relative softmax upper bound on eta nabla L}
{Under the conditions in Proposition \ref{lemma:Properties of PGT estimator}, Lemma \ref{lemma: relative softmax upper bound on I^+}, and Assumption \ref{ass: 3}, suppose that the sequence $\{\theta_t\}_{t=1}^T$ is generated by Algorithm \ref{alg:IS-MBPG-S} with $k>0$, $\lambda>0$, $c=\frac{1}{3k^3L_\lambda} +96b^2$, $m=\max\{2,(2L_\lambda k)^3,(\frac{ck}{2L_\lambda})^3\}$ and $\eta_0=\frac{k}{m^{1/3}}$, where $b^2=L_g^2+G^2C_w^2$, $L_\lambda=L+\lambda$,
and $C_w=\sqrt{H(2HM_g^2+M_h)(W+1)}$. Then, it holds that}
\begin{align} \label{eq: relative softmax upper bound on eta nabla L}
\sum_{t=1}^{T} \EE\left[\eta_t\norm{\nabla L_{\lambda,\mu}^H(\theta_t)}^2_2\right]\leq \Gamma_2 +\frac{\Gamma_1+\Gamma_3}{B},
\end{align}
where 
{\small
\[
\begin{split}
\Gamma_1&=
 \frac{c^2\sigma^2k^3\ln{(T+2)}}{44b^2} +\frac{m^{1/3}\sigma^2}{88b^2k} +\frac{1}{22} \left(L^H_{\lambda,\mu}(\theta^\ast)-L_{\lambda,\mu}^H(\theta_1) \right), \\
 \Gamma_2&=\frac{48}{11}\left(L^H_{\lambda,\mu}(\theta^\ast)-L_{\lambda,\mu}^H(\theta_1) \right),  \\ \Gamma_3&=\frac{\sigma^2m^{1/3}}{44b^2k}+  \frac{c^2 \sigma^2k^3}{22b^2}\ln{(2+T)}.
 \end{split}
 \]
 }
\end{lemma}

Lemma \ref{lemma: relative softmax upper bound on eta nabla L} shows that $\sum_{t=1}^{T} \EE\left[\eta_t\norm{\nabla L_{\lambda,\mu}^H(\theta_t)}^2_2\right]$ is upper-bounded by a constant up to some logarithmic terms.  To prove this result, we first show that $\mathbb{E}[\eta_t\|\nabla L_{\lambda,\mu}^H(\theta_t)\|_2^2]$ can be bounded by two terms, namely the successive iteration differences $\mathbb{E}[\eta_t^{-1}\|\theta_t-\theta_{t-1}\|_2^2]$ and the policy gradient estimation errors $\mathbb{E} [\eta_t||u_t^H -\nabla J^H_\mu(\theta_t)||^2_2]$. The proof then reduces to bounding each of these two terms. In light of the momentum term and the carefully chosen hyper-parameters $\beta_t$ and $\eta_t$, a recursive inequality on the policy gradient estimation errors can be used to make the accumulated policy gradient estimation errors small even with a constant batch size. Finally, the successive iteration differences can be upper-bounded by the smoothness of the regularized value function and the construction of a non-trivial Lyapunov function.
We refer the reader to the appendix in Section \ref{sec: B.2} for the details.

Finally, by combining Lemmas \ref{lemma: relative softmax upper bound on I^+} and \ref{lemma: relative softmax upper bound on eta nabla L}, and using a sufficiently large horizon $H$, we obtain the following global convergence and sample complexity result for Algorithm \ref{alg:IS-MBPG-S}.

\begin{theorem}\label{theorem: relative softmax}
Under the conditions of Lemmas \ref{lemma: relative softmax upper bound on I^+} and \ref{lemma: relative softmax upper bound on eta nabla L}, let
{
$T=\widetilde{\mathcal{O}}\left(\frac{c_\infty^\frac{3}{2}}{\epsilon^{\frac{9}{2}}(1-\gamma)^{\frac{33}{2}}}+  \frac{c_\infty W}{\epsilon^{3}(1-\gamma)^{12}} \right),$
$B=\mathcal{O}(1)$ and
$H=\mathcal{O}\left( \log_{\gamma} \left( \frac{(1-\gamma)\epsilon }{|\mathcal{S}| |\mathcal{A}| }\norm{\frac{d_\rho^{\pi_{\theta^\ast}}}{\mu}}^{-1}_\infty\right) \right)$, where $c_\infty=\norm{\frac{d_\rho^{\pi_{\theta^\ast}}}{\mu}}^{2}_\infty|\mathcal{S}|^2|\mathcal{A}|^2(1+W)$. }
Then, it holds that
$$J_\rho(\theta^\ast)-\frac{1}{T}\EE\left[ \sum_{t=1}^T \left(J_\rho(\theta_t)\right)\right]\leq \epsilon.$$
In total, it requires $\widetilde{O}(\epsilon^{-4.5})$ samples to achieve an $\epsilon$-optimal policy.
\end{theorem}
The detailed proof of Theorem \ref{theorem: relative softmax} can be found in Section \ref{sec: proof of theorem relative softmax} in the appendix. 
\begin{remark}
Theorem \ref{theorem: relative softmax} improves the result of Theorem 6 in \cite{zhang2020sample} from the sample complexity of $\widetilde{\mathcal{O}}(\epsilon^{-6})$ to $\widetilde{\mathcal{O}}(\epsilon^{-4.5})$ for the soft-max parameterization with a log barrier penalty while only using a constant number of batch size $B$.
\end{remark}

\subsection{Fisher-non-degenerate parameterization} \label{sec: Fisher-non-degenerate parameterization}
\subsubsection{Preliminary tools}
 We now study the global convergence of momentum-based policy gradient for the general parameterization satisfying the fisher-non-degenerate assumption in Assumption \ref{ass: fisher info} (Algorithm \ref{alg:IS-MBPG-S}).
 Since this parameterization can be used for general MDPs, it may be restrictive in the sense that it may not contain all stochastic policies and, therefore, may not contain the optimal policy. Thus, there may be some approximation errors.
Our analysis will leverage the notion of \textit{compatible function approximation} in \cite{sutton1999policy} defined as the regression problem:
{\small
\begin{align}\label{eq: compatible function approximation}
\min_{w\in \mathbb{R}^d}\EE_{(s,a)\sim v_\rho^{\pi_{\theta}}} (A^{\pi_\theta}(s,a) - (1-\gamma) w^\top \nabla \log \pi_\theta(a|s) )^2.
\end{align}
}

\begin{algorithm}[H] 
\caption{Momentum-based PG with Fisher-non-degenerate parameterization (STORM-PG-F)}
\label{alg:IS-MBPG-R}
\begin{algorithmic}[1]
\STATE \textbf{Inputs}: Iteration $T$, Horizon $H$, batch size $B$, initial input $\theta_1$, parameters $\{k,m,c\}$ and initial distribution $\rho$;
\STATE \textbf{Outputs}: $\theta_\xi$ chosen uniformly random from $\{\theta_t\}_{t=1}^T$;
\FOR{$t = 1, 2, \dots,T-1$}
\STATE Sample $B$ trajectories $\{\tau^H_i\}_{i=1}^B$ from $p(\cdot|\theta_t,\rho)$;
\IF{$t=1$}
\STATE Compute $u_1^H=\frac{1}{B} \sum_{i=1}^B g(\tau^H_i|\theta_1,\rho)$;
\ELSE{}
\STATE Compute $u_t^H$ based on \eqref{eq: momentum-based PG};
\ENDIF{}
\STATE Compute $\eta_t =\frac{k}{(m+t)^{1/3}}$;
\STATE Update  $\theta_{t+1}=\theta_{t}+\eta_t u_t^H$;
\STATE Update $\beta_{t+1}=c\eta_t^2$;
\ENDFOR{}
\end{algorithmic}
\end{algorithm}

The notion of \textit{compatible function approximation} measures the ability of using the score function $\nabla \log \pi_\theta(a|s)$ as the features to approximate the advantage function $A^{\pi_\theta}(s,a)$. It can be seen that $F_\rho(\theta)^{-1}\nabla J_\rho(\theta)$ is a minimizer of \eqref{eq: compatible function approximation}, due to the first-order optimality conditions. 
Since even the best linear fit using $\nabla \log \pi_\theta(a|s)$ as the features may not perfectly match $A^{\pi_\theta}(s,a)$,  the \textit{compatible function approximation error} may not be $0$ in practice. 
Following the assumptions in \cite{liu2020improved} and \cite{agarwal2020optimality}, we assume that the policy parameterization $\pi_\theta$ achieves an acceptable function approximation, as measured by the \textit{compatible function approximation error} under a shifted distribution $v_\rho^{\pi_{\theta^\ast}}$.

\begin{assumption}
\label{ass: transferred compatible function approximation error}
For every $\theta\in \mathbb{R}^d$, there exists a constant $\epsilon_{\text{bias}}>0$ such that
the transferred compatible function approximation error satisfies 
{\small
 \begin{align*}
    & \EE_{(s,a)\sim v_\rho^{\pi_{\theta^\ast}}} [(A^{\pi_\theta}(s,a) - (1-\gamma) u^{\ast\top} \nabla \log \pi_\theta(a|s) )^2]\leq \epsilon_{\text{bias}},
 \end{align*}
 }
 where $v_\rho^{\pi_{\theta^\ast}}$ is the state-action distribution induced by an optimal policy $\pi_{\theta^\ast}$ that maximizes $J_\rho(\theta)$ and 
 $u^{\ast}\coloneq F_\rho(\theta)^{-1}\nabla J_\rho(\theta)$ is the solution of \eqref{eq: compatible function approximation} . 
 \end{assumption}
 Assumption \ref{ass: transferred compatible function approximation error}, which is also used in \cite{liu2020improved}, means that the parameterization of $\pi_\theta$ makes the advantage function $A^{\pi_\theta}(s,a)$ be able to nearly approximated by using the score function $\nabla \log\pi_\theta (a|s)$ as the features.
When $\pi_\theta$ is a soft-max parameterization, $\epsilon_{bias}$ is 0.
 When $\pi_\theta$ is a rich neural parameterization, $\epsilon_{bias}$ is very small \citep{wang2019neural}.
 \subsubsection{Theoretical results}
Inspired by the global convergence analysis of PG and natural PG in \cite{liu2020improved} and \cite{agarwal2020optimality}, we present a simpler and more general tool for characterizing the global sub-optimality under the Fisher-non-degenerate policy parametrization.  

\begin{lemma}\label{lemma: restricted J start -J t}
Let us consider a general Fisher-non-degenerate policy  $\pi_\theta$ satisfying Assumptions \ref{ass: fisher info}, \ref{ass: 1} and \ref{ass: transferred compatible function approximation error}. 
Then, we have
\begin{align} \label{eq: concise Fisher-non-degenerate bound}
& J_\rho\left(\theta^\ast\right)- J_\rho\left(\theta\right)\leq \frac{\sqrt{\epsilon_{\text{bias}}}}{(1-\gamma)}  + \frac{M_g}{\mu_F } \norm{\nabla J_\rho(\theta)}.
\end{align}
\end{lemma}
Lemma \ref{lemma: restricted J start -J t} relates the global convergence rates of the policy gradient to the transferred compatible function approximation error and the first-order stationary convergence. It can be regarded as the gradient domination condition for Fisher-non-degenerate parametrizations.
{Compared with Proposition 4.5 in \cite{liu2020improved}, our result is more general. In particular, the bound in \eqref{eq: concise Fisher-non-degenerate bound} does not depend on the update rule for $\theta$ and does not require a Lipschitz continuity assumption for the score function $\nabla \log{\pi_\theta(a|s)}$.}

To prove Lemma \ref{lemma: restricted J start -J t}, we first relate the global convergence optimality gap with the stationary convergence of the natural policy gradient $F_\rho(\theta)^{-1}\nabla J_\rho(\theta)$ and the \textit{transferred compatible function approximation error}.
This is achieved by the following two observations: (1) the advantage function $A^{\pi_\theta}(\cdot,\cdot)$ appears in both the performance difference lemma and the definition of the \textit{transferred compatible function approximation error}; (2) the natural policy gradient update is connected with the \textit{transferred compatible function approximation error}.
In light of Assumption \ref{ass: fisher info}, one can relate the first-order stationary convergence of natural policy gradient with the first-order stationary convergence of policy gradient.
For the details, we refer the reader to the appendix in Section \ref{sec: Proof of lemma: restricted J start -J t}.


{
By applying the momentum-based PG with a constant batch size $B$ under the general Fisher-non-degenerate parameterization (see Algorithm \ref{alg:IS-MBPG-R}),
we arrive at the following result:
{
\begin{lemma} \label{lemma: restricted non-adaptive, sum eta e^2}
{Under the conditions in Proposition \ref{lemma:Properties of PGT estimator}, Lemma \ref{lemma: restricted J start -J t}, and Assumption \ref{ass: 3}, suppose that the sequences $\{\theta_t\}_{t=1}^T$ and $\{u^H_t\}_{t=1}^T$ are generated by Algorithm \ref{alg:IS-MBPG-R}. Let 
$k>0$, $c=\frac{1}{3k^3L} +96b^2$, $m=\max\{2,(2Lk)^3,(\frac{ck}{2L})^3\}$ and $\eta_0=\frac{k}{m^{1/3}}$, where $b^2=L_g^2+G^2C_w^2$ and $C_w=\sqrt{H(2HM_g^2+M_h)(W+1)}$. }Then, it holds that
\begin{align}
  &  \frac{1}{T}\sum_{t=1}^{T} \EE  \norm{\nabla J^H_\rho(\theta_t)}_2\leq \sqrt{\frac{\Gamma_2}{k} +\frac{\Gamma_1+\Gamma_3}{kB}} \left(\frac{m^{1/6}}{\sqrt{T}} +\frac{1}{T^{1/3}}\right ) \label{eq: fisher stationary convergence },
\end{align}
where $\Gamma_1=\left(
 \frac{c^2\sigma^2k^3\ln{(T+2)}}{48b^2} +\frac{m^{1/3}}{96b^2k}  \sigma^2+\frac{1}{22(1-\gamma)} \right)$ , $\Gamma_2=\frac{48}{11(1-\gamma)}$, and
$\Gamma_3=\frac{\sigma^2m^{1/3}}{44b^2k^2}+  \frac{c^2 \sigma^2k^3}{22kb^2}\ln{(2+T)}$.
\end{lemma}}
The proof sketch for Lemma \ref{lemma: restricted non-adaptive, sum eta e^2} is similar to that of Lemma \ref{lemma: relative softmax upper bound on eta nabla L} and the detailed proof is provided in Appendix \ref{sec: C.2}.} Finally, by combining Lemmas \ref{lemma: restricted J start -J t} and \ref{lemma: restricted non-adaptive, sum eta e^2}, we obtain the global convergence and sample complexity of Algorithm \ref{alg:IS-MBPG-R}.

\begin{theorem} \label{thm: theorem for general parameter}
Under the conditions of Lemma \ref{lemma: restricted non-adaptive, sum eta e^2}, let $H= \mathcal{O}(\log_{\gamma}\left((1-\gamma)\mu_F \epsilon \right))$, $B=\mathcal{O}(1)$ and {$T=\Tilde{\mathcal{O}}\left( \frac{(1+W)^{\frac{3}{2}}}{\epsilon^3\mu_F^3(1-\gamma)^{12}}+ \frac{1+W}{\epsilon^2 \mu_F^2(1-\gamma)^{11}}+ \frac{W(1+W)}{\epsilon^2 \mu_F^2(1-\gamma)^9}\right)
$}. Then, it holds that
$$
J_\rho(\theta^\ast)-\frac{1}{T} \sum_{t=1}^{T} \EE[J_\rho(\theta_t)]\leq \frac{\sqrt{\varepsilon_{\text {bias }}}}{1-\gamma}+ \epsilon.
$$
In total, it requires $\widetilde{O}(\epsilon^{-3})$ samples to achieve an $(\frac{\sqrt{\varepsilon_{\text {bias }}}}{1-\gamma}+\epsilon)$-optimal policy.
\end{theorem}
{The detailed proof of Theorem \ref{thm: theorem for general parameter} can be found in Section \ref{sec: proof thm: theorem for general parameter} in the appendix. }

\begin{remark}
Theorem \ref{thm: theorem for general parameter} establishes the global convergence of the momentum-based PG proposed in \cite{huang2020momentum}, for which only stationary convergence was previously shown. {The results in Theorem \ref{thm: theorem for general parameter} do not hold for the soft-max parameterization. The reason is that the soft-max parameterization lacks the exploration and thus Assumption \ref{ass: fisher info} is not satisfied.} In addition, 
it improves the result of Theorem 4.6 in \cite{liu2020improved} from the sample complexity of ${\mathcal{O}}(\frac{1}{\epsilon^4})$ to $\widetilde{\mathcal{O}}(\frac{1}{\epsilon^3})$. 
It also improves Theorem 4.11 in \cite{liu2020improved} from using a batch size of $\mathcal{O}(\epsilon^{-1})$ to a constant batch size and from using a double-loop algorithm to a single-loop algorithm, where the later improvement is due to the momentum introduced in \cite{huang2020momentum,yuan2020stochastic}.
\end{remark}


%% file: sec/conclu.tex
\section{Conclusion}\label{sec:conclu}
In this work, we studied the global convergence and the sample complexity of momentum-based stochastic policy gradient methods for both soft-max parameterization and more general parameterization satisfying the fisher-non-degenerate assumption. We showed that adding a momentum improves the global optimally sample
 complexity of vanilla policy gradient methods in both soft-max and Fisher-non-degenerate policy parametereizations with a constant batch size.
  This work provides the first global convergence results for momentum-based policy gradient methods.

{There are also several open problems that may be addressed by combining the techniques introduced in this paper with the existing results in the literature. First, it remains as an open question whether the momentum-based policy gradient can be combined with the natural policy gradient \citep{kakade2001natural} to achieve or exceed the state-of-the-art sample complexity. In addition, it is
desirable to generalize our soft-max setting to the general class of  log-linear policies and remove the “exploration” assumption about the
initial state distribution being component-wise positive.  This may
be achieved by combining our results with the PC-PG and COPOE methods
in \cite{agarwal2020optimality, zanette2021cautiously}.}



%% file: sec/related_work.tex
\aistatstitle{
Supplementary Materials}

\section{Related work.} \label{sec: related work}


\paragraph{Momentum-based policy gradient.}
Conventional approaches to reducing the high variance in PG methods include adding the baselines \citep{sutton1999policy,wu2018variance} and using the actor-critic algorithms \citep{konda2000actor, bhatnagar2009natural, peters2008natural}. The idea of variance reduction, inspired by its successes in the stochastic nonconvex optimization \citep{johnson2013accelerating, allen2016variance,reddi2016stochastic,fang2018spider,nguyen2017sarah}, is also incorporated to improve the PG methods \citep{xu2017stochastic,papini2018stochastic,xu2019sample}. 
In addition, momentum techniques, which are demonstrated as a powerful and generic recipe for accelerating stochastic gradient methods for nonconvex optimization \citep{qian1999momentum, kingma2014adam, reddi2019convergence}, have also been extended to improve PG methods both in theory and in practice \citep{xiong2020non, yuan2020stochastic, pham2020hybrid, huang2020momentum}. 
{\cite{xiong2020non} studies Adam-based policy gradient methods but only achieved $O(\epsilon^{-4})$ sample complexities, which is the same as the vanilla REINFORCE algorithm.
A new STORM-PG method is proposed in \cite{yuan2020stochastic}, which incorporates momentum in the updates and matches the sample complexity of the SRVR-PG method proposed in \cite{xu2019sample} (and also VRMPO) while requiring only single-loop updates and large initialization batches, whereas SRVR-PG and VRMPO require double-loop updates and large batch sizes throughout all iterations. Concurrently, \cite{pham2020hybrid} proposes a hybrid estimator combining the momentum idea with SARAH and considers a more general setting with regularization, and achieves the same $O(\epsilon^{-3})$ sample complexity and again with single-loop updates and large initialization batches. Finally, independently inspired by  STORM algorithm for stochastic optimization in \cite{cutkosky2019momentum}, \cite{huang2020momentum} proposes a class of momentum-based policy gradient algorithms, with  adaptive time-steps, single-loop updates and small batch sizes, which matches the sample complexity in \cite{xu2019sample}.} 

\paragraph{Global convergence of (stochastic) policy gradient.}
The understanding of the PG methods is mostly restricted to their convergence to stationary points of the value function \citep{sutton1999policy, konda2003onactor, papini2018stochastic}. It was not until very recently that a series of works emerged to establish the global convergence properties of these algorithms. \cite{fazel2018global} shows that the linear quadratic regulator problem satisfies a gradient domination condition although it has a nonconvex landscape, implying that the PG methods could converge to the globally optimal policy. \cite{bhandari2019global} generalizes the results in \cite{fazel2018global} from the linear quadratic regulator problem to several control tasks by relating the objective for policy gradient to the objective associated with the Bellman operator. 
For the soft-max parameterization, \cite{mei2020global,mei2021leveraging} show that the value function satisfies a non-uniform \L{ojasiewicz} inequality and a fast global convergence rate can be achieved if the exact PG is available. 
In addition, \cite{agarwal2020optimality} provides a fairly general characterization of global convergence for the PG methods and a sample complexity result for sample-based natural PG updates. By incorporating the variance reduction techniques in the PG methods, an improved sample complexity for the global convergence is established in \cite{liu2020improved} for both PG and natural PG methods.
When overparameterized neural networks are used for function approximation, the global convergence is proved for the (natural) PG methods \citep{wang2019neural} and for the trust-region policy optimization \citep{liu2019neural}. 
Very recently, a series of non-asymptotic global convergence results \citep{hong2020two, xu2020improving, wu2020finite, xu2020non,fu2020single}
have also been established for actor-critic algorithms with (natural) PG or proximal policy optimization used in the actor step. Apart from RL systems with a cumulative sum of rewards, the global convergence results of PG methods for RL systems whose objectives are a general utility function of the state-action occupancy measure are studied in \cite{NEURIPS2020_30ee748d,zhang2021convergence}. In addition, the global convergence of policy-based methods has also been studied in the constrained MDPs \cite{ding2020natural, ying2021dual, efroni2020exploration},  where the optimal policy is usually stochastic and thus policy-based methods are preferred \cite{altman1999constrained}.


\section{Notation} \label{sec: Notation}
The set of real numbers is shown as $\mathbb R$. ${u} \sim \mathcal{U}$  means that $ u$ is a random vector sampled from the distribution $\mathcal{U}$. 
We use $|\mathcal{X}|$ to denote the number of elements in a finite set $X$.
The notions $\mathbb{E}_{\xi}[\cdot]$ and $\mathbb{E}[\cdot]$ refer to the expectation over the random variable $\xi$ and over all of the randomness. The notion $\text{Var}[\cdot]$ refers to the variance.
For vectors ${x}, {y} \in \mathbb{R}^{d}$, let $\|{x}\|_1$, $\|{x}\|_2$ and $\|{x}\|_\infty$ denote the $\ell_{1}$-norm, $\ell_{2}$-norm and $\ell_{\infty}$-norm. We use
$\left\langle {x}, {y}\right\rangle$ to denote the inner product. For a matrix $A$, $A \succcurlyeq0$ means that $A$ is positive semi-definite.
Given a variable $x$, the notation $a=\mathcal{O}(b(x))$ means that $a \leq C \cdot b(x)$ for some constant $C>0$ that is independent of $x$. {Similarly, $a=\widetilde{\mathcal{O}}(b(x))$ indicates that the previous inequality may also depend on the function $\log(x)$, where $C>0$ is again independent of $x$. We use $\mathcal{N}(\mu,\sigma^2)$ to denote a Gaussian distribution with the mean $\mu$ and the variance $\sigma^2$.}

\section{Discussions on the Fisher-non-degenerate parameterization} \label{sec: app-FNP}
The Fisher-non-degenerate setting contains more than the Gaussian policy.
{In particular,  any full-rank exponential family distribution of the form $\pi_{{\theta}}({a}|s) = h({a}|s) \exp\left({\eta}({\theta}|s) \cdot {T}({a}|s)-A({\eta}({\theta}|s)|s)\right)$ under some reasonable conditions is a Fisher-non-degenerate parameterization. More precisely, if Hessian $\nabla^2_{{\eta}} A({\eta}|s)$ is positive definite and Jacobian $\nabla_{{\theta}} {\eta}({\theta}|s)$ has the full-column-rank at all $\theta$,  then the corresponding Fisher information matrix is positive definite. Thus, the commonly used Gaussian policy $\mathcal{N}(f_\theta(s), \Sigma)$ and neural policy $\frac{\exp(f_\theta(s, a))}{\sum_{a^\prime}\exp(f_\theta(s, a^\prime))}$ under the above mentioned conditions are both covered in our Fisher-non-degenerate setting.

In addition, the Fisher-non-degenerate setting implicitly guarantees that the agent is able to explore the state-action space under the considered policy class. Without the non-degenerate Fisher information matrix condition, the global optimum convergence of more general parameterizations would be hard to analyze without introducing the additional exploration procedures in the non-tabular setting. Similar conditions of the Fisher-non-degeneracy are also required in other global optimum convergence framework (Assumption 6.5 in \cite{agarwal2020optimality}) relative condition number condition].}
We kindly refer the reviewer to (Section B.2 in \cite{liu2020improved}) for more discussions on the  Fisher-non-degenerate setting. 

%% file: sec/app_MBPG.tex
\section{Supporting results}

\begin{proposition}[Lemma 1 in \cite{cortes2010learning}]\label{prop: Lemma 1 in cortes2010learning}
Let $w(x)=P(x)/Q(x)$ be the importance weight for two distributions $P$ and $Q$. The following identities hold for the expectation, second moment, and variance of $w(x)$:
\begin{align*}
    \EE[w(x)]=1, \ \EE[w^2(x)] =d_2(P||Q), \ \text{Var}[w(x)]=d_2(P||Q)-1,
\end{align*}
where $d_2(P||Q)=2^{D(P||Q)}$ and $D(P||Q)$ is the \textit{Rényi} divergence between the distributions $P$ and $Q$.
\end{proposition}

\begin{proposition}[Lemma 6.1 in \cite{xu2019sample}]\label{prop: Lemma 6.1 in xu2019sample}
Under Assumptions \ref{ass: 1} and \ref{ass: 3}, let $w(\tau|\theta_{t-1},\theta_t)=p(\tau|\theta_{t-1},\mu)/p(\tau|\theta_t,\mu)$. We have
\begin{align*}
    \text{Var}[w(\tau|\theta_{t-1}, \theta_t)]\leq C_w^2 \norm{\theta_t-\theta_{t-1}}_2^2,
\end{align*}
for any state distribution $\mu$, where $C_w=\sqrt{H(2HM_g^2+M_h)(W+1)}$.
\end{proposition}

\begin{lemma} \label{lemma: due to smoothness}
Suppose that $f(x)$ is $L$-smooth.
Given  $0<\eta_t\leq\frac{1}{2L}$ for all $t\geq 1$, let $\{x_t\}_{t=1}^T$ be generated by a general update of the form $x_{t+1}=x_t+\eta_t u_t$ and let $e_t=u_t-\nabla f(x_t)$. We have
\begin{align*}
f(x_{t+1})\geq& f(x_t)+\frac{\eta_t}{4} \norm{u_t}_2^2-\frac{\eta_t}{2}\norm{e_t}_2^2.
\end{align*}
\end{lemma}
\begin{proof}
Since $f(f)$ is $L$-smooth, one can write
\begin{align*}
& f(x_{t+1})-f(x_t)-\inner{u_t}{x_{t+1}-x_{t}}\\
= &f(x_{t+1})-f(x_t)-\inner{\nabla f(x_t)}{x_{t+1}-x_{t}}+\inner{\sqrt{\eta_t}(\nabla f(x_t)-u_t)}{\frac{1}{\sqrt{\eta_t}}(x_{t+1}-x_{t})}\\
\geq & -\frac{L}{2} \norm{x_{t+1}-x_t}^2-\frac{b\eta_t}{2}\norm{\nabla f(x_t)-u_t}_2^2 -\frac{1}{2b\eta_t}\norm{x_{t+1}-x_t}_2^2\\
=&(-\frac{1}{2b\eta_t}-\frac{L}{2}) \norm{x_{t+1}-x_t}_2^2-\frac{b\eta_t}{2}\norm{e_t}_2^2,
\end{align*}
where the constant $b>0$ is to be determined later. 
By the above inequality and the definition of $x_{t+1}$, we have
\begin{align*}
f(x_{t+1})\geq& f(x_t)+\inner{u_t}{x_{t+1}-x_{t}}-(\frac{1}{2b\eta_t}+\frac{L}{2}) \norm{x_{t+1}-x_t}_2^2-\frac{b\eta_t}{2}\norm{e_t}_2^2\\
=&f(x_t)+\eta_t\norm{u_t}^2-(\frac{\eta_t}{2b}+\frac{L\eta_t^2}{2}) \norm{u_t}_2^2-\frac{b\eta_t}{2}\norm{e_t}_2^2\\
\end{align*}
By choosing $b=1$ and using the fact that $0<\eta_t\leq\frac{1}{2L}$, it holds that
\begin{align*}
f(x_{t+1})\geq& f(x_t)+(\frac{\eta_t}{2}-\frac{L\eta_t^2}{2}) \norm{u_t}_2^2-\frac{\eta_t}{2}\norm{e_t}_2^2\\
\geq& f(x_t)+\frac{\eta_t}{4} \norm{u_t}_2^2-\frac{\eta_t}{2}\norm{e_t}_2^2.
\end{align*}
This completes the proof.
\end{proof}

%% file: sec/appex_softmax_relative_entropy.tex
\section{Proofs of results in Section \ref{sec: Soft-max parameterization with log barrier penalty}} \label{app: softmax relative entropy}

\subsection{Proof of Lemma \ref{lemma: relative softmax upper bound on I^+}} \label{sec: Proof of Lemma relative softmax upper bound on I^+ }
\begin{proof}
We first define the following set of ``bad'' iterates:
\begin{align} \label{eq: definition of I+}
    I^+=\left\{t\in\{1,\ldots,T\}\Big| \norm{\nabla_\theta L_{\lambda,\mu} (\theta_t)}_2 \geq \frac{\lambda}{2|\mathcal{S}| |\mathcal{A}|}\right\},
\end{align}
which counts the number of iterates such that the gradient norms of the KL-regularized objective are large.
Then, one can show that for every $\epsilon>0$ and $\lambda=\epsilon(1-\gamma)/(4\norm{\frac{d_\rho^{\pi_{\theta^\ast}}}{\mu}}_\infty)$, 
we have that $J_\rho(\theta^\ast)-J_\rho(\theta)\leq \frac{\epsilon}{2}$ for all $k\in \{0,\ldots,K\}/ I^+$, while $J_\rho(\theta^\ast)-J_\rho(\theta)\leq 1/(1-\gamma)$ holds trivially for all $k\in I^+$ due to the assumption that the rewards are between 0 and 1. 
Then, by controlling the number of ``bad'' iterates, we obtain the desired optimality guarantee. For simplicity, assume for now that $|I^+|>0$. Since $\eta_t$ is non-increasing in $t$, we have
\begin{align*}
    \sum_{t=1}^T\eta_t  \norm{\nabla L_{\lambda,\mu}(\theta_t)}^2_2\geq &\sum_{t\in I^+}\eta_t  \norm{\nabla L_{\lambda,\mu}(\theta_t)}^2_2\\
    \geq &\frac{\lambda^2}{4|\mathcal{S}|^2 |\mathcal{A}|^2} \sum_{t \in I^+}\eta_t  \\
    \geq &\frac{\lambda^2}{4|\mathcal{S}|^2 |\mathcal{A}|^2} \sum_{t=T-|I^+|+1}^T\eta_t  \\
     \geq &\frac{\lambda^2|I^+|\eta_T}{4|\mathcal{S}|^2 |\mathcal{A}|^2}.
\end{align*}
Thus,
\begin{align*}
  |I^+|\leq \frac{4|\mathcal{S}|^2 |\mathcal{A}|^2\sum_{t=1}^T\eta_t  \norm{\nabla L_{\lambda,\mu}(\theta_t)}_2^2}{\lambda^2\eta_T} . 
\end{align*}
Since $J_\rho(\theta)\in[0,\frac{1}{1-\gamma}]$ for every $\theta$, it holds that $J_\rho(\theta^\ast)- J_\rho(\theta_t)\leq\frac{1}{1-\gamma}$ for all $t\in I^+$. 
In addition, by Lemma \ref{prop: softmax, relative entropy, gradient dom} and the choice of $\lambda=\frac{\epsilon(1-\gamma)}{4\norm{\frac{d_\rho^{\pi_{\theta^\ast}}}{\mu}}_\infty}$, it holds that 
\begin{align*}
J_\rho(\theta^\ast)- J_\rho(\theta_t)\leq \frac{\epsilon}{2},\quad \forall t\notin I^+.
\end{align*}
Therefore, 
\begin{align}\label{eq: relative softmax, upper bound with I^+ }
  \nonumber  \sum_{t=1}^T \left(J_\rho(\theta^\ast)-J_\rho(\theta_t)\right)=&\sum_{t\in I^+} \left(J_\rho(\theta^\ast)-J_\rho(\theta_t)\right)+\sum_{t\notin I^+}\left(J_\rho(\theta^\ast)-J_\rho(\theta_t)\right)\\
   \nonumber   &\leq |I^+| \frac{1}{1-\gamma}+(T-|I^+|) \frac{ \epsilon}{2}\\
    \nonumber  &\leq \frac{4|\mathcal{S}|^2 |\mathcal{A}|^2\sum_{t=1}^T\eta_t  \norm{\nabla L_{\lambda,\mu}(\theta_t)}_2^2}{\lambda^2\eta_T(1-\gamma)}   +\frac{T \epsilon}{2}\\
     &\leq \frac{4|\mathcal{S}|^2 |\mathcal{A}|^2\sum_{t=1}^T\eta_t  \norm{\nabla L_{\lambda,\mu}(\theta_t)}_2^2}{\lambda^2\eta_T(1-\gamma)}   +\frac{T \epsilon}{2}.
\end{align}
Now if $|I^+|=0$, 
\begin{align*}
\sum_{t=1}^T \left(J_\rho(\theta^\ast)-J_\rho(\theta_t)\right)     &\leq \frac{T \epsilon}{2}
\end{align*}
and hence \eqref{eq: relative softmax, upper bound with I^+ } always holds. This completes the proof.
\end{proof}

\subsection{Proof of Lemma \ref{lemma: relative softmax upper bound on eta nabla L}}\label{sec: B.2}
We first notice that Assumption \ref{ass: 1} is satisfied by the soft-max parameterization with $M_g=2$ and $M_h=1$.
\begin{lemma}\label{lemma: mg and mh for softmax}
For the soft-max parameterization, Assumption \ref{ass: 1} is satisfied with $M_g=2$ and $M_h=1$.
\end{lemma}
\begin{proof}
For the soft-max parameterization, we have
\begin{align*}
    \frac{\alpha \log \pi_\theta (a|s)}{\alpha \theta(s,\cdot)}=\boldsymbol{1}_{a}-\EE_{a^\prime \sim \pi_\theta(\cdot|s)}\boldsymbol{1}_{a^\prime},
\end{align*}
where $\boldsymbol{1}_{a}\in\mathbb{R}^{|\mathcal{A}|}$ is a vector with zero entries except one nonzero entry corresponding to the action $a$. In addition, $\frac{\alpha \log \pi_\theta (a|s)}{\alpha \theta(s^\prime,\cdot)}=\boldsymbol{0}$ for all $s\neq s^\prime$. Hence, $\norm{\nabla_\theta \log \pi_\theta (a|s)}_2 \leq 2 $ for every $(s,a)\in\mathcal{S}\times \mathcal{A}$. 

Similarly, we have
\begin{align*}
    \frac{\alpha^2 \log \pi_\theta (a|s)}{\alpha \theta(s,\cdot)^2}=\left(\frac{d \pi_\theta(\cdot|s)}{d \theta(s,\cdot)}\right)^\top = \text{diag}(\pi(\cdot|s))-\pi(\cdot|s)\pi(\cdot|s)^\top.
\end{align*}
From Lemma 22 of \cite{mei2020global}, we know that the largest eigenvalue of the matrix $\text{diag}(\pi(\cdot|s))-\pi(\cdot|s)\pi(\cdot|s)^\top$ is less than 1. Thus, $\norm{\nabla^2_\theta \log \pi_\theta (a|s)}_2 \leq 1 $.
\end{proof}

\begin{lemma}\label{lemma: softmax+KL eta{t-1}^{-1} ||{et}||2}
Suppose that the stochastic policy gradient $u_t^H$ is generated by Algorithm \ref{alg:IS-MBPG-S} with the soft-max parameterization. Let $e^H_t=u_t^H+ \frac{\lambda}{|\mathcal{A}|| \mathcal{S}|}\sum_{s,a} \nabla \log \pi_{\theta_t}(a|s) -\nabla L_{\lambda,\mu}^H(\theta_t)$. It holds that
\begin{align*}
\EE\left[\eta_{t-1}^{-1} ||{e^H_t}||_2^2\right]\leq&\EE\left[\eta_{t-1}^{-1} (1-\beta_t)^2||e^H_{t-1}||_2^2 +\frac{2\eta_{t-1}^{-1}\beta_t^2\sigma^2}{B}+\frac{4b^2\eta_{t-1}^{-1}(1-\beta_t)^2}{B}||\theta_t-\theta_{t-1}||_2^2\right],
\end{align*}
where $b^2=L_g^2+G^2C_w^2$, $L_g=M_h/(1-\gamma)^2$, $G=M_g/(1-\gamma)^2$ and $C_w=\sqrt{H(2HM_g^2+M_h)(W+1)}$.
\end{lemma}
\begin{proof}
First note that $e^H_t=u_t^H- \nabla J^H_\mu(\theta)$. Then,
by the definition of $u_t^H$, we have
\begin{align*}
u_t^H-u_{t-1}^H=-\beta_t u_{t-1}^H +  \frac{\beta_t}{B}\sum_{i=1}^B g(\tau_i^H|\theta_t,\mu) + \frac{(1-\beta_t)}{B}\sum_{i=1}^B( g(\tau_i^H|\theta_t,\mu)-w(\tau_i^H|\theta_{t-1},\theta_t) g(\tau_i^H|\theta_{t-1},\mu)).
\end{align*}
As a result,
\begin{align*}
\begin{split}
\EE\left[\eta_{t-1}^{-1} ||{e^H_t}||_2^2\right]=&\EE\left[\eta_{t-1}^{-1} ||\nabla J_\rho^H(\theta_{t-1})-u_{t-1} +\nabla J_\rho^H(\theta_t)-\nabla J_\rho^H(\theta_{t-1}) -(u_t^H-u_{t-1})  )||_2^2\right]\\
=&\EE\left[\eta_{t-1}^{-1} ||\nabla J_\rho^H(\theta_{t-1})-u_{t-1} +\nabla J_\rho^H(\theta_t)-\nabla J_\rho^H(\theta_{t-1}) \right.\\
&\left.+\beta_t u_{t-1} - \frac{\beta_t}{B}\sum_{i=1}^B g(\tau_i^H|\theta_t,\mu) - \frac{(1-\beta_t)}{B}\sum_{i=1}^B( g(\tau_i^H|\theta_t,\mu)-w(\tau_i^H|\theta_{t-1},\theta_t) g(\tau_i^H|\theta_{t-1},\mu))  )||_2^2\right]\\
=&\EE\left[\eta_{t-1}^{-1} ||(1-\beta_t)(\nabla J_\rho^H(\theta_{t-1})-u_{t-1})+\beta_t (\nabla J_\rho^H(\theta_t)-\frac{1}{B}\sum_{i=1}^B g(\tau_i^H|\theta_t,\mu)) \right.\\
& \left.- \frac{(1-\beta_t)}{B}\sum_{i=1}^B( g(\tau_i^H|\theta_t,\mu)-w(\tau_i^H|\theta_{t-1},\theta_t) g(\tau_i^H|\theta_{t-1},\mu)) -(\nabla J_\rho^H(\theta_t)-\nabla J_\rho^H(\theta_{t-1})) )||_2^2\right]\\
=&\eta_{t-1}^{-1} (1-\beta_t)^2\EE\left[||\nabla J_\rho^H(\theta_{t-1})-u_{t-1}||_2^2\right] +\eta_{t-1}^{-1}\EE\left[ ||\beta_t (\nabla J_\rho^H(\theta_t)-\frac{1}{B}\sum_{i=1}^B g(\tau_i^H|\theta_t,\mu))\right. \\
& \left.- \frac{(1-\beta_t)}{B}\sum_{i=1}^B( g(\tau_i^H|\theta_t,\mu)-w(\tau_i^H|\theta_{t-1},\theta_t) g(\tau_i^H|\theta_{t-1},\mu)) -(\nabla J_\rho^H(\theta_t)-\nabla J_\rho^H(\theta_{t-1})) )||_2^2\right]\\
\leq&\eta_{t-1}^{-1} (1-\beta_t)^2\EE\left[||\nabla J_\rho^H(\theta_{t-1})-u_{t-1}||_2^2\right] +2\eta_{t-1}^{-1}\beta_t^2\EE\left[ || (\nabla J_\rho^H(\theta_t)-\frac{1}{B}\sum_{i=1}^B g(\tau_i^H|\theta_t,\mu))||_2^2\right] \\
&+2\eta_{t-1}^{-1}(1-\beta_t)^2\EE\left[ || \frac{1}{B}\sum_{i=1}^B( g(\tau_i^H|\theta_t,\mu)-w(\tau_i^H|\theta_{t-1},\theta_t) g(\tau_i^H|\theta_{t-1},\mu)) -(\nabla J_\rho^H(\theta_t)-\nabla J_\rho^H(\theta_{t-1})) )||_2^2\right]\\
=&\eta_{t-1}^{-1} (1-\beta_t)^2\EE\left[||e^H_{t-1}||_2^2\right] +2\eta_{t-1}^{-1}\beta_t^2\frac{1}{B}\EE\left[ || g(\tau_i^H|\theta_t,\mu)-\nabla J_\rho^H(\theta_t)||_2^2\right] \\
&+2\eta_{t-1}^{-1}(1-\beta_t)^2\frac{1}{B}\EE\left[ || g(\tau_i^H|\theta_t,\mu)-w(\tau_i^H|\theta_{t-1},\theta_t) g(\tau_i^H|\theta_{t-1},\mu) -(\nabla J_\rho^H(\theta_t)-\nabla J_\rho^H(\theta_{t-1}))||_2^2\right]\\
\leq&\eta_{t-1}^{-1} (1-\beta_t)^2\EE\left[||e^H_{t-1}||_2^2\right] +\frac{2\eta_{t-1}^{-1}\beta_t^2\sigma^2}{B} \\
&+2\eta_{t-1}^{-1}(1-\beta_t)^2\frac{1}{B}\EE\left[ || g(\tau_i^H|\theta_t,\mu)-w(\tau_i^H|\theta_{t-1},\theta_t) g(\tau_i^H|\theta_{t-1},\mu)||_2^2\right],
\end{split}
\end{align*}
where the fourth equality is due to $\EE_{\tau_i^H}[g(\tau_i^H|\theta_t,\mu)]=\nabla J_\rho^H(\theta_t)$ and $\EE_{\tau_i^H}[g(\tau_i^H|\theta_t,\mu)-w(\tau_i^H|\theta_{t-1},\theta_t) g(\tau_i^H|\theta_{t-1},\mu)]=\nabla J_\rho^H(\theta_t)-\nabla J_\rho^H(\theta_{t-1})$, the first inequality follows from Young's inequality, the second inequality holds by $\EE||\frac{1}{B}\sum_{i=1}^B\xi_i-\EE[\xi_i]||_2^2=\frac{1}{B}\EE||\xi_i-\EE[\xi_i]||_2^2$ for the i.i.d. samples of $\{\xi_i\}_{i=1}^B$, and the last inequality is due to the bounded variance of stochastic policy gradient under the soft-max parameterization and $\EE||\xi-\EE[\xi]||_2^2\leq \EE||\xi||_2^2$. 

In addition,
\begin{align*}
&\EE\left[ || g(\tau_i^H|\theta_t,\mu)-w(\tau_i^H|\theta_{t-1},\theta_t) g(\tau_i^H|\theta_{t-1},\mu)||_2^2\right] \\
&= \EE\left[ || g(\tau_i^H|\theta_t,\mu)-g(\tau_i^H|\theta_{t-1},\mu)+g(\tau_i^H|\theta_{t-1},\mu)-w(\tau_i^H|\theta_{t-1},\theta_t) g(\tau_i^H|\theta_{t-1},\mu)||_2^2\right] \\
&\leq 2\EE\left[ || g(\tau_i^H|\theta_t,\mu)-g(\tau_i^H|\theta_{t-1},\mu)||_2^2+2\EE[ ||g(\tau_i^H|\theta_{t-1},\mu)-w(\tau_i^H|\theta_{t-1},\theta_t) g(\tau_i^H|\theta_{t-1},\mu)||_2^2\right] \\
&\leq 2L_g^2\EE\left[ ||\theta_t-\theta_{t-1})||_2^2\right]+2G^2\EE\left[ ||1-w(\tau_i^H|\theta_{t-1},\theta_t)||_2^2\right] \\
&\leq 2L_g^2\EE\left[ ||\theta_t-\theta_{t-1})||_2^2\right]+2G^2\text{Var}(w(\tau_i^H|\theta_{t-1},\theta_t))\\
&\leq 2(L_g^2+G^2C_w^2)\EE\left[ ||\theta_t-\theta_{t-1})||_2^2\right],
\end{align*}
where the second inequality follows from Lemma \ref{lemma:Properties of PGT estimator}, and the third inequality is due to Proposition \ref{prop: Lemma 1 in cortes2010learning}, and the last inequality holds by Proposition \ref{prop: Lemma 6.1 in xu2019sample}.
By selecting $b^2=L_g^2+G^2C_w^2$, we have
\begin{align*}
\EE\left[\eta_{t-1}^{-1} ||{e^H_t}||_2^2\right]\leq&\EE\left[\eta_{t-1}^{-1} (1-\beta_t)^2||e^H_{t-1}||_2^2 +\frac{2\eta_{t-1}^{-1}\beta_t^2\sigma^2}{B}+\frac{4b^2\eta_{t-1}^{-1}(1-\beta_t)^2}{B}||\theta_t-\theta_{t-1}||_2^2\right],
\end{align*}
which completes the proof.
\end{proof}

\subsubsection{Proof of Lemma \ref{lemma: relative softmax upper bound on eta nabla L}}
\begin{proof}
{From Proposition \ref{lemma:Properties of PGT estimator}, we know that $J_{\mu}(\theta)$ is $L$-smooth. Since $\norm{\nabla^2 \log\pi_\theta(a|s)}_2\leq 1$ for the soft-max parameterization, it holds that $L_{\lambda,\mu}(\theta)$ is $L_\lambda$-smooth, where $L_\lambda\coloneqq L+\lambda$.}

Due to $m\geq (2L_\lambda k)^3$, we have $\eta_t\leq\eta_0=\frac{k}{m^{1/3}}\leq\frac{1}{2L_\lambda}$. Since $\eta_t\leq\frac{1}{2L_\lambda}$, we obtain that $\beta_{t+1}=c\eta_t^2\leq\frac{c\eta_t}{2L_\lambda}\leq \frac{ck}{2L_\lambda m^{1/3}}\leq 1$.
Now, it results from Lemma \ref{lemma: softmax+KL eta{t-1}^{-1} ||{et}||2} that
\begin{align*}
 &\EE\left[\eta_{t}^{-1} ||{e^H_{t+1}}||_2^2-\eta_{t-1}^{-1} ||{e^H_t}||_2^2\right]\\
 &\leq\EE\left[(\eta_{t}^{-1} (1-\beta_{t+1})^2-\eta_{t-1}^{-1}) ||e^H_{t}||_2^2 +\frac{2\eta_{t}^{-1}\beta_{t+1}^2\sigma^2}{B}+\frac{4b^2\eta_{t}^{-1}(1-\beta_{t+1})^2}{B}||\theta_{t+1}-\theta_{t}||_2^2\right]\\
 &\leq\EE\left[(\eta_{t}^{-1} (1-\beta_{t+1})-\eta_{t-1}^{-1}) ||e^H_{t}||_2^2 +\frac{2\eta_{t}^{-1}\beta_{t+1}^2\sigma^2}{B}+\frac{4b^2\eta_{t}^{-1}}{B}||\theta_{t+1}-\theta_{t}||_2^2\right],
\end{align*}
where the last inequality holds by $0<\beta_{t+1}\leq 1$. Since the function $x^{1/3}$ is concave, we have $(x+y)^{1/3}\leq x^{1/3}+yx^{-2/3}/3$. Then, we have
\begin{align*}
 \eta_{t}^{-1}-\eta_{t-1}^{-1}  =&\frac{1}{k}((m+t)^{1/3} -(m+t-1)^{1/3})\leq \frac{1}{3k(m+t-1)^{2/3}}\\
 \leq &\frac{1}{3k(m/2+t)^{2/3}}\leq\frac{2^{2/3}}{3k^3}\eta_t^2\leq \frac{2^{2/3}}{6k^3L}\eta_t\leq \frac{1}{3k^3L}\eta_t,
\end{align*}
where the second inequality holds by $m\geq 2$, and the forth inequality uses the property $0<\eta_t\leq \frac{1}{2L_\lambda}$.
Then, it holds that
\begin{align*}
(\eta_{t}^{-1} (1-\beta_{t+1})-\eta_{t-1}^{-1}) ||e^H_{t}||_2^2=&\left(\frac{1}{3k^3L}-c\right)\eta_t\norm{e^H_t}_2^2=-96b^2\eta_t\norm{e^H_t}_2^2,
\end{align*}
where the last equality is based on the relation  $c=\frac{1}{3k^3L_\lambda}+96b^2$. Combining the above results yields that
\begin{align} \label{eq: eta t e t+1-eta t-1 et}
  &\EE[\eta_{t}^{-1} ||{e^H_{t+1}}||_2^2-\eta_{t-1}^{-1} ||{e^H_t}||_2^2]\leq\EE\left[-96b^2\eta_{t} ||e^H_{t}||_2^2 +\frac{2\eta_{t}^{-1}\beta_{t+1}^2\sigma^2}{B}+\frac{4b^2\eta_{t}^{-1}}{B}||\theta_{t+1}-\theta_{t}||_2^2\right].
\end{align}
To streamline the presentation, we denote $u_{t,\lambda}^H\coloneqq \frac{1}{\eta_t}(\theta_{t+1}-\theta_{t})$. 
By summing up the above inequality and dividing the both sides by  $96b^2$, we obtain
\begin{align}\label{eq: et/eta telescope}
\frac{1}{96b^2} \left( \EE\left[\frac{\norm{e^H_{T+1}}_2^2}{\eta_T} - \frac{\norm{e^H_{1}}_2^2}{\eta_{0}}\right]\right)  \leq& \sum_{t=1}^T\EE\left[ \frac{c^2\eta_t^3 \sigma^2}{48b^2B} -\eta_t\norm{e^H_t}_2^2 +\frac{\eta_t}{24B} \norm{u_{t,\lambda}^H}_2^2\right]
\end{align}
For $\eta_t \norm{u_{t,\lambda}^H}_2^2$, it follows from Lemma \ref{lemma: due to smoothness} that
\begin{align*}
 \frac{\eta_t}{4} \norm{u_{t,\lambda}^H}_2^2\leq   L^H_{\lambda,\mu}(\theta_{t+1})-L_{\lambda,\mu}^H(\theta_t) +\frac{\eta_t}{2}  \norm{e_t^H}^2.
\end{align*}

Then, it holds that

\begin{align}\label{eq: et/eta telescope simplified}
\nonumber \frac{1}{96b^2} \left( \EE\left[\frac{\norm{e^H_{T+1}}_2^2}{\eta_T} - \frac{\norm{e^H_{1}}_2^2}{\eta_{0}}\right]\right)  \leq& \sum_{\tau=1}^T\EE\left[ \frac{c^2\eta_t^3 \sigma^2}{48b^2B} -\frac{(12B-1)\eta_t}{12B} \norm{e^H_t}_2^2+ \frac{1}{6B} \left(L^H_{\lambda,\mu}(\theta_{t+1})-L_{\lambda,\mu}^H(\theta_t) \right) \right] \\
\nonumber\leq& \sum_{t=1}^T \left(\frac{c^2\eta_t^3 \sigma^2}{48b^2B} -\EE\left[ \frac{11\eta_t}{12} \norm{e^H_t}_2^2\right]\right)+ \frac{1}{6B} \left(L^H_{\lambda,\mu}(\theta^\ast)-L_{\lambda,\mu}^H(\theta_1) \right) \\
\nonumber\leq&  \frac{c^2 \sigma^2k^3}{48b^2B} \sum_{t=1}^T \frac{1}{m+t} - \sum_{t=1}^T\EE\left[\frac{11\eta_t}{12}\norm{e^H_t}_2^2\right]+ \frac{1}{6B} \left(L^H_{\lambda,\mu}(\theta^\ast)-L_{\lambda,\mu}^H(\theta_1) \right) \\
\nonumber\leq&  \frac{c^2 \sigma^2k^3}{48b^2B} \sum_{t=1}^T \frac{1}{2+t} - \sum_{t=1}^T\EE\left[\frac{11\eta_t}{12}\norm{e^H_t}_2^2\right]+ \frac{1}{6B} \left(L^H_{\lambda,\mu}(\theta^\ast)-L_{\lambda,\mu}^H(\theta_1) \right) \\
\leq&  \frac{c^2\sigma^2k^3\ln{(T+2)}}{48b^2B} - \sum_{t=1}^T\EE\left[\frac{11\eta_t}{12}\norm{e^H_t}_2^2\right]+ \frac{1}{6B} \left(L^H_{\lambda,\mu}(\theta^\ast)-L_{\lambda,\mu}^H(\theta_1) \right).
\end{align}
By rearranging the above inequality, we have
\begin{align*}
 \sum_{t=1}^T\EE\left[\frac{11\eta_t}{12}\norm{e^H_t}_2^2\right]\leq  & \frac{c^2\sigma^2k^3\ln{(T+2)}}{48b^2B} +\frac{1}{96b^2\eta_{0}}  \EE\left[ \norm{e^H_{1}}_2^2\right]+\frac{1}{6B} \left(L^H_{\lambda,\mu}(\theta^\ast)-L_{\lambda,\mu}^H(\theta_1) \right) \\
 \leq &\frac{1}{B}\left(
 \frac{c^2\sigma^2k^3\ln{(T+2)}}{48b^2} +\frac{m^{1/3}\sigma^2}{96b^2k} +\frac{1}{6} \left(L^H_{\lambda,\mu}(\theta^\ast)-L_{\lambda,\mu}^H(\theta_1) \right)  \right).
\end{align*}
Multiplying both sides by $\frac{12}{11}$ yields that
\begin{align*}
  \sum_{t=1}^T\EE\left[\eta_t\norm{e^H_t}_2^2\right]\leq &\frac{1}{B}\left(
 \frac{c^2\sigma^2k^3\ln{(T+2)}}{44b^2} +\frac{m^{1/3}\sigma^2}{88b^2k} +\frac{1}{22} \left(L^H_{\lambda,\mu}(\theta^\ast)-L_{\lambda,\mu}^H(\theta_1) \right)  \right).   
\end{align*}


To bound $\sum_{t=1}^{T}\EE \left[ \eta_t \norm{u_{t,\lambda}^H}_2^2\right]$, we define a Lyapunov function $\Phi_t(\theta_t)=L_{\lambda,\mu}^H(\theta_t)-\frac{1}{192b^2\eta_{t-1}} \norm{e^H_t}_2^2$ for all $t\geq 1$. Then,
\begin{align*}
\EE[ \Phi_{t+1}-\Phi_{t}  ]=&\EE\left[L_{\lambda,\mu}^H(\theta_{t+1}) -L_{\lambda,\mu}^H(\theta_t) -\frac{1}{192b^2\eta_t} \norm{e^H_{t+1} }_2^2+\frac{1}{192b^2 \eta_{t-1}}\norm{e^H_t}_2^2 \right]\\
\geq &\EE\left[ -\frac{\eta_t}{2}\norm{e^H_t}_2^2 +\frac{\eta_t}{4} \norm{u_{t,\lambda}^H }_2^2 -\frac{1}{192b^2\eta_t} \norm{e^H_{t+1} }_2^2+\frac{1}{192b^2 \eta_{t-1}}\norm{e^H_t}_2^2 \right]\\
\geq &\EE\left[\frac{\eta_t}{4} \norm{u_{t,\lambda}^H }_2^2 -\frac{\beta_t^2\sigma^2}{96b^2B\eta_{t}}-\frac{\eta_{t}}{48B}||u_{t,\lambda}^H||^2\right]\\
\geq &\EE\left[\frac{11\eta_t}{48} \norm{u_{t,\lambda}^H }_2^2 -\frac{c^2\eta_{t}^3\sigma^2}{96b^2B}\right],
\end{align*}
where the first inequality holds by Lemma \ref{lemma: due to smoothness} and the second inequality holds due to \eqref{eq: eta t e t+1-eta t-1 et}. Summing the above inequality over $t$ from $1$ to $T$, we obtain
\begin{align*}
 \sum_{t=1}^{T}\EE \left[ \eta_t \norm{u_{t,\lambda}^H}_2^2\right]
 \leq&\EE\left[\frac{48}{11}(\Phi_{T+1} -\Phi_{1}) + \sum_{t=1}^{T} \frac{c^2\eta_t^3 \sigma^2}{22b^2B}\right]\\
    \leq &\EE\left[\frac{48}{11}\left(L^H_{\lambda,\mu}(\theta^\ast)-L_{\lambda,\mu}^H(\theta_1) \right)  +\frac{1}{44b^2\eta_0}\EE||e^H_1||_2^2+  \frac{c^2 \sigma^2k^3}{22b^2B}\sum_{t=1}^{T}\frac{1}{m+t}  \right]\\
   \leq &\EE\left[\frac{48}{11}\left(L^H_{\lambda,\mu}(\theta^\ast)-L_{\lambda,\mu}^H(\theta_1) \right)  +\frac{1}{44b^2\eta_0}\EE||e^H_1||_2^2+  \frac{c^2 \sigma^2k^3}{22b^2B}\sum_{t=1}^{T}\frac{1}{2+t}  \right]\\
 \leq &\EE\left[\frac{48}{11}\left(L^H_{\lambda,\mu}(\theta^\ast)-L_{\lambda,\mu}^H(\theta_1) \right)  +\frac{\sigma^2m^{1/3}}{44b^2kB}+  \frac{c^2 \sigma^2k^3}{22b^2B}\ln{(2+T)} \right ]\\
  \leq &\Gamma_2 +\frac{\Gamma_3}{B} ,
\end{align*}
where $\Gamma_2=\frac{48}{11}\left(L^H_{\lambda,\mu}(\theta^\ast)-L_{\lambda,\mu}^H(\theta_1) \right) $ and  $\Gamma_3=(\frac{\sigma^2m^{1/3}}{44b^2k}+  \frac{c^2 \sigma^2k^3}{22b^2}\ln{(2+T)})$. 
Finally, by the triangle inequality, we have
\begin{align*}
  \sum_{t=1}^{T} \EE \left[\eta_t \norm{\nabla L_{\lambda,\mu}^H(\theta_t)}_2^2\right]
  \leq &  \sum_{t=1}^{T} \EE \left[\eta_t \norm{u_{t,\lambda}^H}_2^2\right]+\sum_{t=1}^{T} \EE \left[\eta_t \norm{e^H_t}_2^2\right]\\
  \leq & \Gamma_2 +\frac{\Gamma_1+\Gamma_3}{B}.
\end{align*}
This completes the proof.
\end{proof}

\subsection{Proof of Theorem \ref{theorem: relative softmax}} \label{sec: proof of theorem relative softmax}
\begin{proof}
From Proposition \ref{lemma:Properties of PGT estimator}, we have
\begin{align*}
\sum_{t=1}^T\eta_t  \EE\left[\norm{\nabla L_{\lambda,\mu}(\theta_t)}_2^2\right]=&\sum_{t=1}^T\eta_t  \EE\left[\norm{\nabla L^H_{\lambda,\mu}(\theta_t)}_2^2+\norm{\nabla L^H_{\lambda,\mu}(\theta_t)-\nabla L_{\lambda,\mu}(\theta_t)}_2^2\right] \\
\leq&\sum_{t=1}^T\eta_t  \EE\left[\norm{\nabla L^H_{\lambda,\mu}(\theta_t)}_2^2+\norm{\nabla J^H_{\mu}(\theta_t)-\nabla J_{\mu}(\theta_t)}_2^2\right] \\
\leq&\sum_{t=1}^T\eta_t  \EE\left[\norm{\nabla L^H_{\lambda,\mu}(\theta_t)}_2^2\right] + \left(M_g\left(\frac{H+1}{1-\gamma} +\frac{\gamma}{(1-\gamma)^2}\right) \gamma^H \right)^2\sum_{t=1}^T\eta_t.
\end{align*}
In light of Lemma \ref{lemma: relative softmax upper bound on I^+}, we have
\begin{align*}
\EE\left[ \sum_{t=1}^T \left(J_\rho(\theta^\ast)-J_\rho(\theta_t)\right)\right]
&\leq \norm{\frac{d_\rho^{\pi_{\theta^\ast}}}{\mu}}^2_\infty\frac{64|\mathcal{S}|^2 |\mathcal{A}|^2\sum_{t=1}^T\eta_t  \EE[ \norm{\nabla L_{\lambda,\mu}(\theta_t)}_2^2]}{\epsilon^2\eta_T(1-\gamma)^3} +\frac{T\epsilon}{2}\\
&\leq \norm{\frac{d_\rho^{\pi_{\theta^\ast}}}{\mu}}^2_\infty\frac{64|\mathcal{S}|^2 |\mathcal{A}|^2\sum_{t=1}^T\eta_t  \EE[ \norm{\nabla L^H_{\lambda,\mu}(\theta_t)}_2^2]}{\epsilon^2\eta_T(1-\gamma)^3} +\frac{T\epsilon}{2}\\
&+\norm{\frac{d_\rho^{\pi_{\theta^\ast}}}{\mu}}^2_\infty\frac{64|\mathcal{S}|^2 |\mathcal{A}|^2\sum_{t=1}^T\eta_t }{\epsilon^2\eta_T(1-\gamma)^3} \left(M_g\left(\frac{H+1}{1-\gamma} +\frac{\gamma}{(1-\gamma)^2}\right) \gamma^H \right)^2.
\end{align*}
By choosing $\eta_t=\frac{k}{(m+t)^{1/3}}$, it results from Lemma \ref{lemma: relative softmax upper bound on eta nabla L} that
\begin{align*}
\frac{1}{T}\EE\left[ \sum_{t=1}^T \left(J_\rho(\theta^\ast)-J_\rho(\theta_t)\right)\right]
&\leq \norm{\frac{d_\rho^{\pi_{\theta^\ast}}}{\mu}}^2_\infty\frac{64|\mathcal{S}|^2 |\mathcal{A}|^2(m+T)^{1/3}}{\epsilon^2k(1-\gamma)^3 T}\left(\Gamma_2+\frac{\Gamma_1+\Gamma_3}{B}\right) +\frac{\epsilon}{2}\\
&+\norm{\frac{d_\rho^{\pi_{\theta^\ast}}}{\mu}}^2_\infty\frac{96|\mathcal{S}|^2 |\mathcal{A}|^2(m+T) }{\epsilon^2(1-\gamma)^3T} \left(M_g\left(\frac{H+1}{1-\gamma} +\frac{\gamma}{(1-\gamma)^2}\right) \gamma^H \right)^2.
\end{align*}

Notice that, in order to guarantee 
\begin{align*}
  \norm{\frac{d_\rho^{\pi_{\theta^\ast}}}{\mu}}^2_\infty\frac{96|\mathcal{S}|^2 |\mathcal{A}|^2(m+T) }{\epsilon^2(1-\gamma)^3T} \left(M_g\left(\frac{H+1}{1-\gamma} +\frac{\gamma}{(1-\gamma)^2}\right) \gamma^H \right)^2 \leq \frac{\epsilon}{4},
\end{align*}
it suffices to have
\begin{align*}
H={\mathcal{O}}\left( \log_{\gamma} \left( \frac{(1-\gamma)\epsilon }{|\mathcal{S}| |\mathcal{A}| \norm{\frac{d_\rho^{\pi_{\theta^\ast}}}{\mu}}_\infty}\right) \right).
\end{align*}
Recall that 
$\Gamma_1=
 \frac{c^2\sigma^2k^3\ln{(T+2)}}{44b^2} +\frac{m^{1/3}\sigma^2}{88b^2k} +\frac{1}{22} \left(L^H_{\lambda,\mu}(\theta^\ast)-L_{\lambda,\mu}^H(\theta_1) \right)$, $\Gamma_2=\frac{48}{11}\left(L^H_{\lambda,\mu}(\theta^\ast)-L_{\lambda,\mu}^H(\theta_1) \right) $ and  $\Gamma_3=(\frac{\sigma^2m^{1/3}}{44b^2k}+  \frac{c^2 \sigma^2k^3}{22b^2}\ln{(2+T)})$. By only considering the dependencies on the parameters $c, \sigma, b, \lambda$ and $m$, we have  
 \begin{align}\label{eq: big O of Gamma}
\nonumber\Gamma_2+\frac{\Gamma_1+\Gamma_3}{B}    =&\widetilde{\mathcal{O}}\left(\frac{c^2\sigma^2}{b^2}+\frac{m^{1/3}\sigma^2}{b^2} + \frac{\sigma^2m^{1/3}}{b^2}+ \frac{c^2\sigma^2}{b^2}+L^H_{\lambda,\mu}(\theta^\ast)-L_{\lambda,\mu}^H(\theta_1)\right)\\
     =&\widetilde{\mathcal{O}}\left(\frac{\sigma^2}{b^2} \left(c^2+m^{1/3} \right)+L^H_{\lambda,\mu}(\theta^\ast)-L_{\lambda,\mu}^H(\theta_1)\right).
 \end{align}
 In addition,  by the definitions $b^2=L_g^2+G^2C_w^2$, $L_g=M_h/(1-\gamma)^2$, $G=M_g/(1-\gamma)^2$, $C_w=\sqrt{H(2HM_g^2+M_h)(W+1)}$, $c=\frac{1}{3k^3L_\lambda} +96b^2$, $m=\max\{2,(2L_\lambda k)^3,(\frac{ck}{2L_\lambda})^3\}$, and $\lambda=\frac{\epsilon(1-\gamma)}{4\norm{\frac{d_\rho^{\pi_{\theta^\ast}}}{\mu}}_\infty}$,
 we know that 
 \begin{align*}
  b^2={\mathcal{O}}\left(\frac{1+W}{(1-\gamma)^4}\right), \ \sigma^2={\mathcal{O}}\left(\frac{1}{(1-\gamma)^4}\right), \ c={\mathcal{O}}\left(\frac{1+W}{(1-\gamma)^4} \right), \ m^{1/3}={\mathcal{O}}\left(\frac{\epsilon(1-\gamma)}{\norm{\frac{d_\rho^{\pi_{\theta^\ast}}}{\mu}}_\infty}+\frac{1}{(1-\gamma)^3} + \frac{W}{1-\gamma}\right).
 \end{align*}
In addition,
$$L^H_{\lambda,\mu}(\theta^\ast)-L_{\lambda,\mu}^H(\theta_1) \leq \mathcal{O}\left( \frac{1}{1-\gamma}+\frac{\epsilon(1-\gamma)}{\norm{\frac{d_\rho^{\pi_{\theta^\ast}}}{\mu}}_\infty}\right).$$
 Substituting  the above results into \eqref{eq: big O of Gamma} gives rise to 
\begin{align*}
\Gamma_2+\frac{\Gamma_1+\Gamma_3}{B}
=&\widetilde{\mathcal{O}}\left(\frac{1+W}{(1-\gamma)^8} + \frac{\epsilon(1+W)(1-\gamma)}{\norm{\frac{d_\rho^{\pi_{\theta^\ast}}}{\mu}}_\infty} \right).
\end{align*}
 Thus, we obtain
\begin{align*}
&\hspace{-1.5cm}\frac{1}{T}\EE\left[ \sum_{t=1}^T \left(J_\rho(\theta^\ast)-J_\rho(\theta_t)\right)\right]\\
&\leq \widetilde{\mathcal{O}} \left(\norm{\frac{d_\rho^{\pi_{\theta^\ast}}}{\mu}}^2_\infty\frac{|\mathcal{S}|^2 |\mathcal{A}|^2(m+T)^{1/3}(1+W)}{\epsilon^2(1-\gamma)^{11}T}\right)
+\frac{3\epsilon}{4}\\
&\leq \norm{\frac{d_\rho^{\pi_{\theta^\ast}}}{\mu}}^2_\infty|\mathcal{S}|^2 |\mathcal{A}|^2(1+W)\cdot\widetilde{\mathcal{O}} \left(  \frac{ m^{1/3}}{\epsilon^2(1-\gamma)^{11}T} +\frac{1}{\epsilon^2(1-\gamma)^{11}T^{2/3}} \right)+\frac{3\epsilon}{4}.
\end{align*}
Since $m^{1/3}={\mathcal{O}}\left(\frac{\epsilon(1-\gamma)}{\norm{\frac{d_\rho^{\pi_{\theta^\ast}}}{\mu}}_\infty}+\frac{1}{(1-\gamma)^3} + \frac{W}{1-\gamma}\right)$,  one can write 
\begin{align*}
&\hspace{-1.5cm}\frac{1}{T}\EE\left[ \sum_{t=1}^T \left(J_\rho(\theta^\ast)-J_\rho(\theta_t)\right)\right]\\
&\leq \norm{\frac{d_\rho^{\pi_{\theta^\ast}}}{\mu}}^2_\infty|\mathcal{S}|^2 |\mathcal{A}|^2(1+W)\cdot\widetilde{\mathcal{O}} \left(\frac{1}{\epsilon^2(1-\gamma)^{14}T}+\frac{ W}{\epsilon^2(1-\gamma)^{12}T}  +\frac{1}{\epsilon^2(1-\gamma)^{11}T^{2/3}} \right)+\frac{3\epsilon}{4}.
\end{align*}
Finally, to guarantee $\frac{1}{T}\EE\left[ \sum_{t=1}^T \left(J_\rho(\theta^\ast)-J_\rho(\theta_t)\right)\right]\leq \epsilon$, 
it suffices to have 
\begin{align*}
T=\widetilde{\mathcal{O}}\left(\frac{\norm{\frac{d_\rho^{\pi_{\theta^\ast}}}{\mu}}^3_\infty|\mathcal{S}|^3|\mathcal{A}|^3(1+W)^\frac{3}{2}}{\epsilon^{\frac{9}{2}}(1-\gamma)^{\frac{33}{2}}}+ \frac{\norm{\frac{d_\rho^{\pi_{\theta^\ast}}}{\mu}}^{2}_\infty|\mathcal{S}|^2|\mathcal{A}|^2(1+W)}{\epsilon^{3}} \cdot\frac{W}{(1-\gamma)^{12}} \right).
\end{align*}

This completes the proof.
\end{proof}

%% file: sec/app_compatible.tex
\section{Proofs of results in Section \ref{sec: Fisher-non-degenerate parameterization}}\label{app:restricted}

\subsection{Proof of Lemma \ref{lemma: restricted J start -J t}} \label{sec: Proof of lemma: restricted J start -J t}

\begin{proof}
By the performance difference lemma \cite{kakade2002approximately}, we know that 
\begin{align} \label{eq: performance difference lemma}
\mathbb{E}_{(s,a)\sim v_\rho^{\pi_{\theta^\ast}}}\left[A^{\pi_{\theta_t}}(s, a)\right]=(1-\gamma)\left(J_\rho(\theta^\star)-J_\rho\left(\theta_{t}\right)\right).
\end{align}

In addition, by Assumption \ref{ass: transferred compatible function approximation error}, we know that the advantage function is also related to the defined \textit{transferred compatible function approximation error} that measures the richness of the policy parameterization:

 \begin{align} \label{eq: epiislob bias in proof}
\nonumber \epsilon_{\text{bias}} 
\geq & \EE_{(s,a)\sim v_\rho^{\pi_{\theta^\ast}}} [(A^{\pi_{\theta_t}}(s,a) - (1-\gamma) u_t^{\ast \top} \nabla \log \pi_{\theta_t}(a|s) )^2]\\
\geq & \left(  \EE_{(s,a)\sim v_\rho^{\pi_{\theta^\ast}}} [A^{\pi_{\theta_t}}(s,a) - (1-\gamma) u_t^{\ast \top} \nabla \log \pi_{\theta_t}(a|s) ] \right)^2
 \end{align}
 where the second inequality uses the Jensen's inequality.
 Then, by combining \eqref{eq: performance difference lemma} and \eqref{eq: epiislob bias in proof}, we have
 \begin{align*}
     \sqrt{\epsilon_{\text{bias}}}\geq  & \EE_{(s,a)\sim v_\rho^{\pi_{\theta^\ast}}} [A^{\pi_{\theta_t}}(s,a) - (1-\gamma) u_t^{\ast \top} \nabla \log \pi_{\theta_t}(a|s) ] \\
     =& (1-\gamma)\left(J_\rho(\theta^\star)-J_\rho\left(\theta_{t}\right)\right) - \EE_{(s,a)\sim v_\rho^{\pi_{\theta^\ast}}} [(1-\gamma) u_t^{\ast \top} \nabla \log \pi_{\theta_t}(a|s)].
 \end{align*}
The rearrangement of the above inequality gives 
 \begin{align*}
\left(J_\rho(\theta^\star)-J_\rho\left(\theta_{t}\right)\right)\leq & \frac{1}{(1-\gamma)} \sqrt{\epsilon_{\text{bias}}} + \EE_{(s,a)\sim v_\rho^{\pi_{\theta^\ast}}} [ u_t^{\ast \top} \nabla \log \pi_{\theta_t}(a|s)]\\
\leq&  \frac{1}{(1-\gamma)} \sqrt{\epsilon_{\text{bias}}} + \EE_{(s,a)\sim v_\rho^{\pi_{\theta^\ast}}} [ \norm{u_t^{\ast}}  \norm{\nabla \log \pi_{\theta_t}(a|s)}]\\
\leq & \frac{1}{(1-\gamma)} \sqrt{\epsilon_{\text{bias}}} + M_g \norm{u_t^{\ast}}.
 \end{align*}
 
In addition, by the definition of $u_t^{\ast}$, we have
 \begin{align*}
J_\rho(\theta^\star)-J_\rho\left(\theta_{t}\right)
\leq & \frac{1}{(1-\gamma)} \sqrt{\epsilon_{\text{bias}}} + M_g \norm{F^{-1}(\theta_t) \nabla J_\rho(\theta_t)} \\
\leq & \frac{1}{(1-\gamma)} \sqrt{\epsilon_{\text{bias}}} + \frac{M_g}{\mu_F} \norm{\nabla J_\rho(\theta_t)},
 \end{align*}
 where the second inequality follows from Assumption \ref{ass: fisher info}. This completes the proof.

\end{proof}

\subsection{Proof of Lemma \ref{lemma: restricted non-adaptive, sum eta e^2}} \label{sec: C.2}
\begin{lemma}\label{lemma: restricted eta{t-1}^{-1} ||{et}||2}
Under Assumption \ref{ass: 1}, suppose that the stochastic policy gradient $u_t^H$ is generated by Algorithm \ref{alg:IS-MBPG-R} with the restricted parameterization. Let  $e^H_t=\nabla J_{\rho}^H(\theta_t)-u_t^H$. Then
\begin{align*}
\EE\left[\eta_{t-1}^{-1} ||{e^H_t}||_2^2\right]\leq&\EE\left[\eta_{t-1}^{-1} (1-\beta_t)^2||e^H_{t-1}||_2^2 +\frac{2\eta_{t-1}^{-1}\beta_t^2\sigma^2}{B}+\frac{4b^2\eta_{t-1}^{-1}(1-\beta_t)^2}{B}||\theta_t-\theta_{t-1}||_2^2\right],
\end{align*}
where $b^2=L_g^2+G^2C_w^2$, $L_g=M_h/(1-\gamma)^2$, $G=M_g/(1-\gamma)^2$ and $C_w=\sqrt{H(2HM_g^2+M_h)(W+1)}$.
\end{lemma}

\begin{proof}
This proof is similar to the proof of Lemma \ref{lemma: softmax+KL eta{t-1}^{-1} ||{et}||2} with $M_g$ and $M_h$ defined in Assumption \ref{ass: 1}. The details are omitted for brevity. 
\end{proof}

\subsubsection{Proof of Lemma \ref{lemma: restricted non-adaptive, sum eta e^2}}
\begin{proof}
Let $e_t^H=\nabla J_\rho^H(\theta_t)-u_t^H$. The function $J^H_\rho(\theta)$ is $L$-smooth due to Lemma \ref{lemma:Properties of PGT estimator}. Moreover, because of $m\geq (2Lk)^3$, it holds that $\eta_t\leq\eta_0=\frac{k}{m^{1/3}}\leq\frac{1}{2L}$. Since $\eta_t\leq\frac{1}{2L}$, we have $\beta_{t+1}=c\eta_t^2\leq\frac{c\eta_t}{2L}\leq \frac{ck}{2Lm^{1/3}}\leq 1$.
It follows from Lemma \ref{lemma: restricted eta{t-1}^{-1} ||{et}||2} that
\begin{align*}
 &\EE\left[\eta_{t}^{-1} ||{e^H_{t+1}}||_2^2-\eta_{t-1}^{-1} ||{e^H_t}||_2^2\right]\\
 &\leq\EE\left[(\eta_{t}^{-1} (1-\beta_{t+1})^2-\eta_{t-1}^{-1}) ||e^H_{t}||_2^2 +\frac{2\eta_{t}^{-1}\beta_{t+1}^2\sigma^2}{B}+\frac{4b^2\eta_{t}^{-1}(1-\beta_{t+1})^2}{B}||\theta_{t+1}-\theta_{t}||_2^2\right]\\
 &\leq\EE\left[(\eta_{t}^{-1} (1-\beta_{t+1})-\eta_{t-1}^{-1}) ||e^H_{t}||_2^2 +\frac{2\eta_{t}^{-1}\beta_{t+1}^2\sigma^2}{B}+\frac{4b^2\eta_{t}^{-1}}{B}||\theta_{t+1}-\theta_{t}||_2^2\right],
\end{align*}
where the last inequality holds by $0<\beta_{t+1}\leq 1$. Since the function $x^{1/3}$ is concave, we have $(x+y)^{1/3}\leq x^{1/3}+yx^{-2/3}/3$. As a result
\begin{align*}
 \eta_{t}^{-1}-\eta_{t-1}^{-1}  =&\frac{1}{k}\left((m+t)^{1/3} -(m+t-1)^{1/3}\right)\leq \frac{1}{3k(m+t-1)^{2/3}}\\
 \leq &\frac{1}{3k(m/2+t)^{2/3}}\leq\frac{2^{2/3}}{3k^3}\eta_t^2\leq \frac{2^{2/3}}{6k^3L}\eta_t\leq \frac{1}{3k^3L}\eta_t,
\end{align*}
where the second inequality is due to $m\geq 2$, and the fifth inequality holds by $0<\eta\leq \frac{1}{2L}$.
Then, it holds that
\begin{align*}
\left(\eta_{t}^{-1} (1-\beta_{t+1})-\eta_{t-1}^{-1}\right) ||e^H_{t}||_2^2=&\left(\frac{1}{3k^3L}-c\right)\eta_t\norm{e^H_t}_2^2=-96b^2\eta_t\norm{e^H_t}_2^2,
\end{align*}
where the last equality holds by $c=\frac{1}{3k^3L}+96b^2$. Combining the above results leads to
\begin{align} \label{eq: restricted eta t e t+1-eta t-1 et}
  &\EE\left[\eta_{t}^{-1} ||{e^H_{t+1}}||_2^2-\eta_{t-1}^{-1} ||{e^H_t}||_2^2\right]\leq\EE\left[-96b^2\eta_{t} ||e^H_{t}||_2^2 +\frac{2\eta_{t}^{-1}\beta_{t+1}^2\sigma^2}{B}+\frac{4b^2\eta_{t}^{-1}}{B}||\theta_{t+1}-\theta_{t}||_2^2\right].
\end{align}
By summing up the above inequality and dividing both sides by $96b^2$, we obtain
\begin{align*} 
\frac{1}{96b^2} \left( \EE\left[\frac{\norm{e^H_{T+1}}_2^2}{\eta_T} - \frac{\norm{e^H_{1}}_2^2}{\eta_{0}}\right]\right)  \leq& \sum_{t=1}^T\EE\left[ \frac{c^2\eta_t^3 \sigma^2}{48b^2B} -\eta_t\norm{e^H_t}_2^2 +\frac{1}{24\eta_tB} \norm{\theta_{t+1}-\theta_t}_2^2\right]\\
 \leq& \sum_{t=1}^T\EE\left[ \frac{c^2\eta_t^3 \sigma^2}{48b^2B} -\eta_t\norm{e^H_t}_2^2 +\frac{\eta_t}{24B} \norm{u_t^H}_2^2\right].
\end{align*}
For $\eta_t \norm{u_{t}^H}_2^2$, it follows from Lemma \ref{lemma: due to smoothness} that
\begin{align*}
 \frac{\eta_t}{4} \norm{u_{t,\lambda}^H}_2^2\leq   J^H_{\rho}(\theta_{t+1})-J_{\rho}^H(\theta_t) +\frac{\eta_t}{2}  \norm{e_t^H}^2.
\end{align*}

Then, it holds that

Then, Lemma \ref{lemma: due to smoothness} can be used to obtain
\begin{align} 
&\nonumber \frac{1}{96b^2} \left( \EE\left[\frac{\norm{e^H_{T+1}}_2^2}{\eta_T} - \frac{\norm{e^H_{1}}_2^2}{\eta_{0}}\right]\right)\\
&\nonumber\leq \sum_{\tau=1}^T\EE\left[ \frac{1}{48b^2B}c^2\eta_t^3 \sigma^2 -\frac{(12B-1)\eta_t}{12B} \norm{e^H_t}_2^2+ \frac{1}{6B} \left(J_\rho^H(\theta_{t+1})-J_\rho^H(\theta_{t}) \right)\right] \\
&\nonumber\leq\sum_{t=1}^T \left(\frac{1}{48b^2B}c^2\eta_t^3 \sigma^2 -\EE\left[ \frac{11\eta_t}{12} \norm{e^H_t}_2^2\right]\right)+ \frac{1}{6B} \left(J_\rho^H(\theta_{T+1})-J_\rho^H(\theta_{1}) \right)\\
&\nonumber\leq  \frac{c^2 \sigma^2k^3}{48b^2B} \sum_{t=1}^T \frac{1}{m+t} - \sum_{t=1}^T\EE\left[\frac{11\eta_t}{12}\norm{e^H_t}_2^2\right]+ \frac{1}{6B} \left(J_\rho^H(\theta^\ast)-J_\rho^H(\theta_{1}) \right)\\
&\nonumber\leq  \frac{c^2 \sigma^2k^3}{48b^2B} \sum_{t=1}^T \frac{1}{2+t} - \sum_{t=1}^T\EE\left[\frac{11\eta_t}{12}\norm{e^H_t}_2^2\right]+ \frac{1}{6B} \left(J_\rho^H(\theta^\ast)-J_\rho^H(\theta_{1}) \right)\\
&\leq  \frac{c^2\sigma^2k^3\ln{(T+2)}}{48b^2B} - \sum_{t=1}^T\EE\left[\frac{11\eta_t}{12}\norm{e^H_t}_2^2\right]+ \frac{1}{6B} \left(J_\rho^H(\theta^\ast)-J_\rho^H(\theta_{1}) \right).
\end{align}
Rearranging the above inequality gives rise to
\begin{align*}
 \sum_{t=1}^T\EE\left[\frac{11\eta_t}{12}\norm{e^H_t}_2^2\right]\leq  & \frac{c^2\sigma^2k^3\ln{(T+2)}}{48b^2B} +\frac{1}{96b^2\eta_{0}}  \EE\left[ \norm{e^H_{1}}_2^2\right]+\frac{1}{6B} \left(J_\rho^H(\theta^\ast)-J_\rho^H(\theta_{1}) \right)\\
 \leq &\frac{1}{B}\left(
 \frac{c^2\sigma^2k^3\ln{(T+2)}}{48b^2} +\frac{m^{1/3}\sigma^2}{96b^2k} +\frac{1}{6} \left(J_\rho^H(\theta^\ast)-J_\rho^H(\theta_{1}) \right) \right).
\end{align*}
Multiplying both sides by $\frac{12}{11}$ yields that
\begin{align} \label{eq: restrcited eta e square}
  \sum_{t=1}^T\EE\left[\eta_t\norm{e^H_t}_2^2\right]\leq &\frac{1}{B}\left(
 \frac{c^2\sigma^2k^3\ln{(T+2)}}{44b^2} +\frac{m^{1/3}\sigma^2}{88b^2k} +\frac{1}{22} \left(J_\rho^H(\theta^\ast)-J_\rho^H(\theta_{1}) \right) \right)\\
 \leq &\frac{1}{B}\left(
 \frac{c^2\sigma^2k^3\ln{(T+2)}}{44b^2} +\frac{m^{1/3}\sigma^2}{88b^2k} +\frac{1}{22(1-\gamma)}\right),
\end{align}
where the last inequality holds due to $J_\rho^H(\theta^\ast)-J_\rho^H(\theta_{1}) \leq \frac{1}{1-\gamma}$.
Next, we define a Lyapunov function $\Phi_t(\theta_t)=J_\rho^H(\theta_t)-\frac{1}{192b^2\eta_{t-1}} \norm{e^H_t}^2$ for all $t\geq 1$. One can write
\begin{align*}
\EE[ \Phi_{t+1}-\Phi_{t}  ]=&\EE\left[J_\rho^H(\theta_{t+1}) -J_\rho^H(\theta_t) -\frac{1}{192b^2\eta_t} \norm{e^H_{t+1} }_2^2+\frac{1}{192b^2 \eta_{t-1}}\norm{e^H_t}_2^2 \right]\\
\geq &\EE\left[ -\frac{\eta_t}{2}\norm{e^H_t}_2^2 +\frac{\eta_t}{4} \norm{u^H_{t} }_2 -\frac{1}{192b^2\eta_t} \norm{e^H_{t+1} }_2^2+\frac{1}{192b^2 \eta_{t-1}}\norm{e^H_t}_2^2 \right]\\
\geq &\EE\left[\frac{\eta_t}{4} \norm{u^H_{t} }_2 -\frac{\beta_t^2\sigma^2}{96b^2B\eta_{t}}-\frac{\eta_{t}}{48B}||u^H_{t}||_2^2\right]\\
\geq &\EE\left[\frac{11\eta_t}{48} \norm{u^H_{t} }_2 -\frac{c^2\eta_{t}^3\sigma^2}{96b^2B}\right]
\end{align*}
where the first inequality holds by Lemma \ref{lemma: due to smoothness} and the second inequality follows from \eqref{eq: restricted eta t e t+1-eta t-1 et}. Summing the above inequality over $t$ from $1$ to $T$ yields that
\begin{align*}
\sum_{t=1}^{T} \EE \left[\eta_t \norm{u_t^H}_2^2\right]
 \leq&\EE\left[\frac{48}{11}(\Phi_{T+1} -\Phi_{1}) + \sum_{t=1}^{T} \frac{c^2\eta_t^3 \sigma^2}{22b^2B}\right]\\
    \leq &\EE\left[\frac{48}{11}(J_\rho^H(\theta_{T+1}) -J_\rho^H(\theta_{1})) +\frac{1}{44b^2\eta_0}\EE\left[||e^H_1||_2^2\right]+  \frac{c^2 \sigma^2k^3}{22b^2B}\sum_{t=1}^{T}\frac{1}{m+t}  \right]\\
   \leq &\EE\left[\frac{48}{11}(J_\rho^H(\theta_{T+1}) -J_\rho^H(\theta_{1})) +\frac{1}{44b^2\eta_0}\EE\left[||e^H_1||_2^2\right]+  \frac{c^2 \sigma^2k^3}{22b^2B}\sum_{t=1}^{T}\frac{1}{2+t}  \right]\\
 \leq &\frac{48}{11}\left(J_\rho^H(\theta^\ast) -J_\rho^H(\theta_{1})\right) +\frac{\sigma^2m^{1/3}}{44b^2kB}+  \frac{c^2 \sigma^2k^3}{22b^2B}\ln{(2+T)}  \\
  \leq &\frac{1}{1-\gamma}\frac{48}{11} +\frac{\sigma^2m^{1/3}}{44b^2kB}+  \frac{c^2 \sigma^2k^3}{22b^2B}\ln{(2+T)}  \\
=  &\Gamma_2 +\frac{\Gamma_3}{B}.
\end{align*}
where the last inequality is due to $J_\rho^H(\theta^\ast) -J_\rho^H(\theta_{1})\leq \frac{1}{1-\gamma}$. This completes the proof.

It results from \eqref{eq: restrcited eta e square} that
\begin{align*}
  \sum_{t=1}^{T} \EE \left[\eta_t \norm{\nabla J^H_\rho(\theta_t)}_2^2\right]
  \leq   \sum_{t=1}^{T} \EE \left[\eta_t \norm{u_t^H}_2^2\right]+\sum_{t=1}^{T} \EE \left[\eta_t \norm{e^H_t}_2^2\right]
  \leq  \Gamma_2 +\frac{\Gamma_1+\Gamma_3}{B}.
\end{align*}
 Since $\eta_t=\frac{k}{(m+t)^{1/3}}$ is decreasing, we have
\begin{align*}
 \sum_{t=1}^{T} \EE \left[ \norm{\nabla J^H_\rho(\theta_t)}_2^2\right]\leq& 1/\eta_T \sum_{t=1}^{T} \EE \left[\eta_t \norm{\nabla J^H_\rho(\theta_t)}_2^2\right]\\
 =&\left(\frac{\Gamma_2}{k} +\frac{\Gamma_1+\Gamma_3}{kB}\right)(m+T)^{1/3}.
\end{align*}

Finally, one can use Jensen's inequality to conclude that
\begin{align*}
    \frac{1}{T} \sum_{t=1}^{T} \EE \left[\norm{\nabla J^H_\rho(\theta_t)}_2\right] \leq & \left(\frac{1}{T} \sum_{t=1}^{T} \EE \norm{\nabla J^H_\rho(\theta_t)}_2^2\right)^{1/2}\\
    \leq & \sqrt{\frac{\Gamma_2}{k} +\frac{\Gamma_1+\Gamma_3}{kB}} \left(\frac{m^{1/6}}{\sqrt{T}} +\frac{1}{T^{1/3}}\right ),
\end{align*}
where the last inequality follows from the inequality $(a+b)^{1/2}\leq a^{1/2}+b^{1/2}$ for all $a, b>0$.

\end{proof}

\subsection{Proof of Theorem \ref{thm: theorem for general parameter}} \label{sec: proof thm: theorem for general parameter}
\begin{proof}
By Lemma \ref{lemma: restricted J start -J t} and the triangle inequality, we have
\begin{align} \label{eq: general parameter optimality}
\nonumber J_\rho\left(\theta^\ast\right)-\frac{1}{T} \sum_{t=1}^{T} J_\rho\left(\theta_{t}\right) & \leq \frac{\sqrt{\varepsilon_{\text {bias }}}}{1-\gamma} + \frac{M_g}{\mu_F} \frac{1}{T} \sum_{t=1}^{T}\norm{\nabla J^H_\rho(\theta_t)} + \frac{M_g}{\mu_F}\max_{t=1,\ldots,T} \left\{  \norm{\nabla J^H_\rho(\theta_t)-\nabla J_\rho(\theta_t)} \right\}\\
& \leq \frac{\sqrt{\varepsilon_{\text {bias }}}}{1-\gamma} + \frac{M_g}{\mu_F} \frac{1}{T} \sum_{t=1}^{T}\norm{\nabla J^H_\rho(\theta_t)} + \frac{M_g^2}{\mu_F} \left(\frac{H+1}{1-\gamma} +\frac{\gamma}{(1-\gamma)^2}\right) \gamma^H,
\end{align}
where the last inequality follows from Proposition \ref{lemma:Properties of PGT estimator}.
Then, due to Lemma \ref{lemma: restricted non-adaptive, sum eta e^2}, we know that 
\begin{align*} 
 \nonumber J_\rho\left(\theta^\ast\right)-\frac{1}{T} \sum_{t=1}^{T} J_\rho\left(\theta_{t}\right) & \leq \frac{\sqrt{\varepsilon_{\text {bias }}}}{1-\gamma} + \frac{M_g}{\mu_F} \sqrt{\frac{\Gamma_2}{k} +\frac{\Gamma_1+\Gamma_3}{kB}} \left(\frac{m^{1/6}}{\sqrt{T}} +\frac{1}{T^{1/3}}\right ) + \frac{M_g^2}{\mu_F} \left(\frac{H+1}{1-\gamma} +\frac{\gamma}{(1-\gamma)^2}\right) \gamma^H,
\end{align*}
Notice that, in order to guarantee 
\begin{align*}
 \frac{M_g^2}{\mu_F} \left(\frac{H+1}{1-\gamma} +\frac{\gamma}{(1-\gamma)^2}\right) \gamma^H \leq \frac{\epsilon}{2},
\end{align*}
it suffices to have
\begin{align*}
H={\mathcal{O}}\left( \log_{\gamma} \left( \mu_F(1-\gamma)\epsilon \right) \right).
\end{align*}
Recall that 
$\Gamma_1=\left(
 \frac{c^2\sigma^2k^3\ln{(T+2)}}{48b^2} +\frac{m^{1/3}}{96b^2k}  \sigma^2+\frac{1}{22(1-\gamma)} \right)$ , $\Gamma_2=\frac{48}{11(1-\gamma)}$, and
$\Gamma_3=\frac{\sigma^2m^{1/3}}{44b^2k^2}+  \frac{c^2 \sigma^2k^3}{22kb^2}\ln{(2+T)}$. By only considering the dependencies on the parameters $c, \sigma, b, \gamma$ and $m$, we have  
\begin{align} \label{eq: big O of Gamma general}
\nonumber\frac{\Gamma_2}{k} +\frac{\Gamma_1+\Gamma_3}{kB} =&\widetilde{\mathcal{O}}\left(\frac{c^2\sigma^2}{b^2}+\frac{m^{1/3}\sigma^2}{b^2} + \frac{\sigma^2m^{1/3}}{b^2}+ \frac{c^2\sigma^2}{b^2}+\frac{1}{1-\gamma}\right)\\
     =&\widetilde{\mathcal{O}}\left(\frac{\sigma^2}{b^2} \left(c^2+m^{1/3} \right)+\frac{1}{1-\gamma}\right).
 \end{align}
 In addition,  by the definitions $b^2=L_g^2+G^2C_w^2$, $L_g=M_h/(1-\gamma)^2$, $G=M_g/(1-\gamma)^2$, $C_w=\sqrt{H(2HM_g^2+M_h)(W+1)}$, $c=\frac{1}{3k^3L} +96b^2$, and $m=\max\{2,(2L k)^3,(\frac{ck}{2L})^3\}$,
 we know that 
 \begin{align*}
  b^2={\mathcal{O}}\left(\frac{1+W}{(1-\gamma)^4}\right), \ \sigma^2={\mathcal{O}}\left(\frac{1}{(1-\gamma)^4}\right), \ c={\mathcal{O}}\left(\frac{1+W}{(1-\gamma)^4} \right), \ m^{1/3}={\mathcal{O}}\left(\frac{1}{(1-\gamma)^3} + \frac{W}{1-\gamma}\right).
 \end{align*}
Substituting the above results into \eqref{eq: big O of Gamma general} gives rise to
\begin{align*}
\frac{\Gamma_2}{k} +\frac{\Gamma_1+\Gamma_3}{kB}
=&\widetilde{\mathcal{O}}\left(\frac{1+W}{(1-\gamma)^8} \right).
\end{align*}
 Thus, we obtain
 \begin{align*} 
 \nonumber J_\rho\left(\theta^\ast\right)-\frac{1}{T} \sum_{t=1}^{T} J_\rho\left(\theta_{t}\right) 
 & \leq \frac{\sqrt{\varepsilon_{\text {bias }}}}{1-\gamma} + \frac{M_g}{\mu_F}  \left(\frac{m^{1/6}}{\sqrt{T}} +\frac{1}{T^{1/3}}\right )\cdot \widetilde{\mathcal{O}}\left(\frac{\sqrt{1+W}}{(1-\gamma)^4} \right) + \frac{\epsilon}{2}\\
 & \leq \frac{\sqrt{\varepsilon_{\text {bias }}}}{1-\gamma} +   \widetilde{\mathcal{O}}\left(\left(\frac{1}{\sqrt{(1-\gamma)^3T}} +\frac{\sqrt{W}}{\sqrt{(1-\gamma)T}} +\frac{1}{T^{1/3}}\right )\frac{\sqrt{1+W}}{(1-\gamma)^4\mu_F} \right) + \frac{\epsilon}{2}
\end{align*}
In order to guarantee 
\begin{align*}
    \widetilde{\mathcal{O}}\left(\left(\frac{1}{\sqrt{(1-\gamma)^3T}} +\frac{\sqrt{W}}{\sqrt{(1-\gamma)T}} +\frac{1}{T^{1/3}}\right )\frac{\sqrt{1+W}}{(1-\gamma)^4\mu_F} \right)\leq \frac{\epsilon}{2},
\end{align*}
it suffices to take 
\begin{align*}
   T=\Tilde{\mathcal{O}}\left( \frac{(1+W)^{\frac{3}{2}}}{\epsilon^3\mu_F^3(1-\gamma)^{12}}+ \frac{1+W}{\epsilon^2 \mu_F^2(1-\gamma)^{11}}+ \frac{W(1+W)}{\epsilon^2 \mu_F^2(1-\gamma)^9}\right).
\end{align*}
This completes the proof.
\end{proof}